\documentclass[10pt]{article}
\usepackage{amsmath,amsthm,amssymb,amsfonts,cancel}
\usepackage{thmtools,enumitem}
\usepackage{algorithm}
\usepackage[noend]{algpseudocode}
\usepackage{mathtools}
\usepackage{dsfont}
\usepackage{tcolorbox}
\usepackage{bbm}
\usepackage{todonotes}
\usepackage{tensor}
\usepackage{caption}
\usepackage{bm}
\usepackage{caption}
\usepackage{hyperref}
\usepackage{graphicx}
\usepackage{multirow}
\usepackage{subfigure}
\usepackage{afterpage}
\usepackage{float}

\renewcommand{\P}{\mathcal{P}}

\newcommand{\Z}{\mathcal{Z}}

\renewcommand{\Pr}{\mathbb{P}}

\newcommand{\K}{\mathcal{K}}
\newcommand{\G}{\mathcal{G}}

\newcommand{\ignore}[1]{}

\newcommand{\mysecx}{\vspace{.02cm}}
\newcommand{\mysec}[1]{\mysecx {\bf #1: }}

\mathchardef\mhyphen="2D % Define a "math hyphen"
\DeclareMathOperator*{\argmin}{\arg\!\min}

\DeclarePairedDelimiter{\abs}{\lvert}{\rvert}

\let\originalleft\left
\let\originalright\right
\renewcommand{\left}{\mathopen{}\mathclose\bgroup\originalleft}
\renewcommand{\right}{\aftergroup\egroup\originalright}

\newtheorem{definition}{Definition}
\usepackage{comment}
\renewcommand{\baselinestretch}{1.5}

\usepackage{graphicx}
\usepackage{subfigure}
\usepackage{booktabs} % for professional tables
\newtheorem{theorem}{Theorem}[section]

\usepackage[margin=1in]{geometry}
\usepackage{adjustbox}

\begin{document}
	\title{Fair %and Interpretable 
	Decision Rules for Binary Classification}
	\author{Oktay G\"unl\"uk \quad Connor Lawless \\ School of Operations Research and Information Engineering \\ Cornell University}
	\date{\today}

	\maketitle
	
	\begin{abstract}
In recent years, machine learning has begun automating decision making in fields as varied as college admissions, credit lending, and criminal sentencing. The socially sensitive nature of some of these applications together with increasing regulatory constraints has necessitated the need for algorithms that are both fair and interpretable. In this paper we consider the problem of building Boolean rule sets in disjunctive normal form (DNF), an interpretable model for binary classification, subject to fairness constraints. We formulate the problem as an integer program that maximizes classification accuracy with explicit constraints on two different measures of classification parity: equality of opportunity and equalized odds. Column generation framework, with a novel formulation, is used to efficiently search over exponentially many possible rules. When combined with faster heuristics, our method can deal with large data-sets. Compared to other fair and interpretable classifiers, our method is able to find rule sets that meet stricter notions of fairness with a modest trade-off in accuracy.
	\end{abstract}

%\newpage\tableofcontents\newpage

\section{Introduction}
With the explosion of artificial intelligence in recent years, machine learning (ML) has begun taking over key decision making tasks in a variety of areas ranging from finance to driving. While in some applications the objective of the ML model is purely predictive accuracy, in others, such as lending or hiring, additional considerations come into play. In particular, when the decision making task at hand has a societal impact, a natural question is whether or not the ML model is being fair to all those affected. Recent results have shown ML algorithms to be racially biased in a number of applications including facial identification in picture tagging and predicting criminal recidivism \cite{mehrabi2019survey}. Further complicating the problem is the need for model interpretability in certain classification applications where ML models complement human decision making, such as criminal justice and medicine. In these applications transparency is necessary for domain experts to understand, critique, and consequently trust the ML models. With these dueling objectives in mind, practitioners need to design classification algorithms that are accurate, fair AND interpretable. %This paper takes one step towards such an algorithm for supervised binary classification problems using integer programming to build interpretable rule sets that can explicitly include fairness constraints. 
    
Many common classification models can easily be expressed as mathematical optimization problems. Discrete optimization is a natural tool for these problems as many interpretable machine learning models can be represented by low-complexity discrete objects. So far, integer programming (IP), a promising technique within discrete optimization that has been successful in many industrial applications such as production planning, scheduling, and logistics, has been sparsely used in machine learning applications. An important reason behind this is the presence of big datasets in ML applications which lead to large-scale IPs that are considered to be computationally intractable. However with the help of algorithmic advances combined with improvements in computational hardware IP methods are becoming an increasingly promising approach for certain ML problems such as optimal decision trees \cite{aghaei2019learning, bertsimas2017, molero2020, gunluk2019optimal}, risk scores \cite{ustun2019learning}, and rule sets \cite{dash2018boolean} amongst others.

In this paper, we focus on a well-studied interpretable class of ML models for binary classification, namely rule sets in disjunctive normal form (DNF, 'OR-of-ANDs'). For example, a DNF rule set with two rules for predicting criminal recidivism could be 
\begin{center}
\big[\text{(Priors$\ge$3) and (Age$\le$45) and (Score Factor $=$ TRUE)}\big] \\ \text{~OR~}
\\\big[\text{(Priors$\ge$20) and (Age$\ge$45)}\big] 
\end{center}

\noindent where Priors, Age, and Score Factor are features related to the defendant. This rule set has two clauses, each of which check certain conditions on the features of the data.
The fewer the clauses or conditions in each clause, the more interpretable the rule set. In contrast to other interpretable classes of rule sets such as decision trees \cite{aghaei2019learning, bertsimas2017, breiman1984, kamiranDADT, quinlan1993},  and decision lists \cite{angelino2017, lakkaraju2017, letham2015, rivest1987,  wang2015, yang2017} the rules within a DNF rule set are unordered and have been shown in a user study to require less effort to understand \cite{lakkarajuDecisionSets2016}. In recent years, discrete optimization approaches to learning DNF rule sets have seen impressive results in empirical interpretable machine learning, including winning the FICO interpretable machine learning challenge \cite{FICO2018, dash2018boolean}. However there has been little work to extend these frameworks to include fairness constraints - even though many applications requiring interpretability are socially sensitive in nature, including the credit approval application in the FICO challenge.

\subsection{Our contribution}
In this paper we present an integer programming formulation to build Boolean DNF rule sets subject to explicit constraints on two notions of fairness: equality of opportunity and equalized odds \cite{hardt2016equality}. Our framework also allows for explicit bounds on the complexity (i.e. number of clauses and conditions) of the rule set, and by extension it's interpretability. Unlike other popular fair classification algorithms, we directly incorporate the original fairness constraints into our formulation instead of using relaxed versions of the criteria. The addition of both the fairness and complexity constraints allows our approach to generate rule sets that trade off accuracy, interpretability, and fairness. 

Instead of directly optimizing for rule set accuracy our formulation optimizes for Hamming loss, a proxy for 0-1 classification error. This new objective function leads to a significantly smaller IP formulation which can be solved much faster while approximating the 0-1 objective closely. We use a column generation (CG) framework to search over the candidate rules, only considering rules that can improve predictive accuracy subject to the fairness constraints. To generate candidate clauses we use an integer programming formulation that solves the pricing problem. 

To help scale our method to larger data sets, we mine rules from heuristic tree-based methods as a warm-start operation on our column generation procedure. In addition we incorporate other heuristics, such as sub-sampling data points and features in the pricing problem, to further speed up solve time. These computational ideas enable our framework to generate high-performing rule sets in a fraction of the time of other IP-based approaches such as optimal decision trees. 

We also present extensive computational studies on our method, and compare it to other state of the art fair and interpretable machine learning algorithms. In addition, we provide empirical analysis of Hamming loss as a proxy for 0-1 loss, and compare our column generation approach to other heuristic rule mining strategies. 

\subsection{Related work}

Our work builds upon two areas of related work: discrete optimization for interpretable ML, and fair ML.

\textbf{Discrete optimization for machine learning: } The last few years have seen a renewed interest in using discrete optimization to solve machine learning problems. Ustun and Rudin have explored the use of IP methods to create sparse linear models for classification \cite{ustun2014supersparse} and risk scores \cite{ustun2019learning}. A number of researchers have also approached the problem of constructing decision trees \cite{aghaei2019learning, bertsimas2017, molero2020, gunluk2019optimal}. Most of these approaches aim to establish a certificate of optimality at the expense of computational effort, requiring hours of computation time.

Recent work in constructing rule sets has relied on rule miners to generate lists of candidate clauses \cite{lakkaraju2016, wangf2015, wang2017}. In contrast, our approach jointly generates candidate clauses and selects the optimal rule set. Closely related to our approach, Dash et al. \cite{dash2018boolean} also formulate constructing DNF rule sets as an IP and generate candidate clauses using column generation. Compared to Dash et al., our formulation includes  the fairness constraints together with some additional constraints that are not necessary without fairness constraints. In addition, we use a more compact formulation for the pricing problem and reduce the number of constraints by a factor of the dimensionality of the feature space. We also present a different approach to handle the upper bounds in the linear relaxation that prevents cycling which might happen with the formulation used by  Dash et al.. Computationally, we also differ in the use of heuristic rule mining from tree-based algorithms as a warm start to our column generation process.

\textbf{Fair machine learning: } Quantifying fairness is not a straight forward task and a number of metrics have been proposed in the fair machine learning literature. These metrics broadly fall into three groups: disparate treatment, classification parity, and calibration \cite{corbettdavies2018measure}. The first measure, disparate treatment, simply excludes sensitive features (i.e. race, gender), or proxies of these sensitive attributes, from the data - however removing sensitive attributes can lead to sub-optimal predictive performance \cite{corbettdavies2018measure}. Classification parity, or group fairness, ensures that some measure of prediction error (ex. Type I/II error, accuracy) is equal across all groups. Recent results have built fair classifiers around various related metrics including demographic parity \cite{agarwal2018reductions, calders2010, dwork2011fairness, edwards2015censoring, kamiran2012}, equalized odds \cite{agarwal2018reductions, hardt2016equality, Zafar_2017}, and equality of opportunity \cite{hardt2016equality, Zafar_2017}. We note that these metrics are only meaningful if the algorithm is well calibrated (i.e. once a prediction has been made, it should mean the same thing for each individual regardless of group). More importantly, recent impossibility results  \cite{Chouldechova_2017, kleinberg2016inherent} have shown that simultaneously attaining  perfect calibration and certain measures of classification parity is not possible. Thus, the chief goal of a fair classifier is to maximize predictive accuracy subject to some requirement on fairness. 
    
%Researchers have developed approaches to incorporate fairness throughout the pipeline of creating a classifier including pre-processing the data to remove any intrinsic bias in the data before applying standard classification algorithms  \cite{dwork2011fairness, pmlr-v28-zemel13, kamrianPreprocess}, integrating fairness considerations directly into the algorithm\cite{zafar2015fairness, kamishima2011, wu2018fairnessaware}, and post-processing the output of a classifier to make it fair \cite{kamiranDADT, hardt2016equality, fish2016confidencebased}. While all three approaches have received attention in recent years, empirical results have shown that pre-processing methods may lead to sub-optimal accuracy \cite{corbettdavies2018measure}, and post-processing methods fail to leverage information about the classification algorithm itself.

Our work focuses on  explicitly integrating fairness considerations directly into the training of a classification model. Previous work in this area has focused on adding either some form of fairness regularization to the loss function \cite{berk2017convex, pmlr-v28-zemel13}, or as a constraint in the underlying optimization problem \cite{aghaei2019learning, zafar2015fairness, Zafar_2017}. However many current approaches require the use of a relaxed version of the fairness constraints (i.e. convex, linear) during optimization \cite{donini2018empirical, wu2018fairnessaware, Zafar_2017}, which have been shown to have sub-par fairness on out-of-sample data \cite{pmlr-v119-lohaus20a}. Similar to our approach, Aghaei et al. \cite{aghaei2019learning} formulate optimal decision trees subject to explicit constraints on fairness. However, unlike our approach which uses heuristics to speed up the solve time, their approach aims to solve the MIP formulation to optimality, requiring multiple hours to construct a tree as opposed to the five minute time limit we place on our approach.

\subsection{Outline of the paper}
Section \ref{sec:fairMet} gives an overview of the two notions of fairness metrics we consider. Section \ref{sec:class} introduces our MIP formulation for constructing boolean rule sets and the column generation procedure to generate candidate clauses. We discuss computational tricks to improve the speed and scalability of our framework in section \ref{sec:comp}, and present empirical results in section \ref{sec:exp}. Comprehensive experimental results are left for the appendix.

\section{Fairness metrics}\label{sec:fairMet}
%=======================================================================

As a starting point we consider the standard supervised binary classification setting where given {a training set of $n$} data points $(\textbf{X}_i, y_i)$ with labels $y_i \in \{-1,1\}$ and features $X_i \in \{0,1\}^{p}$  for $i\in I=\{1,\ldots,n\}$, the goal  is to design a classifier $d:  \{0,1\}^p \rightarrow \{-1,1\}$  that minimizes the expected error $\Pr(d(X)\neq Y)$
%\begin{equation*}\Pr(d(X)\neq Y)\end{equation*} 
between the predicted label and the true label for unseen data.
Note that assuming the data to be binary-valued is not a restrictive assumption in practice and 
we discuss how to deal with numerical and categorical features in Section \ref{sec:comp}. %expand on this point, seems a liittle light 

Now consider the case when each data point also has an associated group (or protected feature) $g_i\in\mathcal{G}$ where $\mathcal{G}$ is a given discrete set and the classifier is not only required to predict the labels well, but it is also required to treat each group fairly. In particular, we will focus on two measures of group fairness related to classification parity: equality of opportunity, and equalized odds.
%We now define two measures of fairness related to classification parity. 

\subsection{Equality of opportunity} 
%=================================
Equality of opportunity requires the Type II error rate (i.e. false negative rate) to be equal across groups by enforcing the following condition \cite{hardt2016equality}: 	
	\begin{equation} 
    \Pr(d(X) = -1 | Y = 1, G = g) = 
    \Pr(d(X) = -1 | Y = 1)    \label{eq:EOO}
    \end{equation}
for all $ g \in \mathcal{G} $. In other words, the false negative rate of the classifier is required to be independent of the group the data point belongs to. This criterion ensures fairness when there is a much larger societal cost to false negatives than false positives, making it particularly well-suited for applications such as loan approval or hiring decisions. For example, in the context of hiring, it ensures qualified candidates would be offered a job with equal probability, independent of their group membership (ex. male/female).
    
\subsection{Equalized odds} 
%=================================
A stricter condition on the classifier is to require that both the Type I and Type II error rates are equal across groups \cite{hardt2016equality}.  
This requirement prevents possible trade-off between false negative and false positive errors across groups and can be seen as a generalization of the equality of opportunity criterion to include false positives. 
To achieve equalized odds, together with equation \eqref{eq:EOO}, the following condition is also enforced:
   	\begin{equation} 
    \Pr(d(X) = 1 | Y = -1, G = g) = \Pr(d(X) = 1 | Y = -1)   \label{eq:EO}
    \end{equation}
for all $ g\in \mathcal{G}$. This notion of fairness equalizes the error rate between the groups for both Type I and Type II errors, but may be overly restrictive in applications where there is little difference in the societal cost between the two types of errors.

\mysecx
In a practical setting, it is unrealistic to expect to find classifiers that can satisfy the above criteria exactly. In fact, in most non-trivial applications strong adherence to fairness criteria come at a large cost to accuracy \cite{kleinberg2016inherent} and therefore one needs to consider how much these conditions are violated as a measure of fairness. For example, in the context of equality of opportunity, the maximum violation can be used to measure the {\em unfairness} of the classifier by the following expression:
\smallskip
\begin{align*}   \Delta(d)=  \max_{g, g' \in \mathcal{G}} \Big|\Pr(d(X)& = -1 | Y = 1, G = g) - \Pr(d(X) = -1 | Y = 1, G = g') \Big|
   \end{align*}

When training the classifier $d$, one can then use $  \Delta(d)$ in the objective function as a penalty term or can explicitly require a constraint of the form $\Delta(d)\le\epsilon$ to be satisfied by the classifier $d$. We will focus on the latter case as it allows for explicit control over tolerable unfairness.

\section{Classification framework: boolean decision rule sets}\label{sec:class}
%==========================================================

For the remainder of this paper, we focus on one specific type of classifier: decision rule sets in DNF (OR-of-ANDs). Given a dataset with a collection of features, a decision rule consists of a subset of the features each with an associated condition (i.e. Age $\geq 25$). A rule is satisfied by a data point if it meets all of the conditions included in the rule. A decision rule set is composed of a collection of rules and a data point satisfies the rule set if it meets any of its rules. We refer to a set of rules that could be included in a rule set as a pool of candidate rules.

We now introduce our method to construct DNF rule sets for binary classification subject to fairness constraints, henceforth referred to as Fair CG. Note that when the input data is binary-valued, a DNF-rule set simply corresponds to checking whether a subset of features satisfies a specific combination of 0s and 1s. Moreover if each data point includes the complement of every feature, the rule set simplifies even further and only needs to check if a subset of features are all 1 for a given data point.
Consequently, if there are $p$ binary features there can only be  a finite number ($2^p-1$) of possible decision rules. 
Therefore, it is possible to enumerate all possible rules and then formulate a large scale integer program (IP) to select a small subset of rules that minimizes error on the training data. 
In this framework, it is also possible to explicitly require the rule set to satisfy certain properties such as fairness or interpretability.
However, for most practical applications such an IP would be onerously large and computationally intractable. Instead, we solve the continuous relaxation (LP) of the IP using column generation. 
Consequently, instead of enumerating all possible rules, one can enumerate those that can potentially improve classification error. 	
  
 \subsection{0-1 loss}
 When constructing a rule set, our ultimate aim is to minimize 0-1 classification error. Let $\K$ denote the set of all possible DNF rules and $\K_i\subset\mathcal{K}$ be the set of rules met by data point $i\in I$. Assume that the data points are partitioned into two sets based on their labels: 
	\begin{equation*}\mathcal{P} = \{i\in I : y_i = 1\},\text{~~and~~}\mathcal{N} = \{i\in I : y_i = -1\}.\end{equation*}
We denote ${\cal P}$ and ${\cal N}$ the positive and negative classes respectively. Furthermore, let $K \subseteq \mathcal{K}$ represent a DNF rule set composed of candidate rules from $\mathcal{K}$. For data points from the positive class (i.e. $i \in {\cal P}$), the 0-1 loss is simply the indicator that the data point meets no rules in rule set (i.e. $|\mathcal{K}_i \cap K| = \emptyset$). For points in the negative class, the 0-1 loss is the indicator of whether the data point meets at least one rule in the rule set (i.e. $|\mathcal{K}_i \cap K| > 0$). Putting both terms together, we get that the 0-1 loss for a data point $(x_i, y_i)$ and rule set $K$ is as follows:

\begin{definition}[0-1 loss]
$$
\ell_{01}(x_i,y_i,K) = \mathbb{I}(i \in {\cal P}) \mathbb{I}(|\mathcal{K}_i \cap K| = 0) + \mathbb{I}(i \in {\cal N}) \mathbb{I}(|\mathcal{K}_i \cap K| > 0)
$$
\end{definition}

Let  $w_{k}\in\{0,1\}$  be a  variable indicating if rule $k\in\K$ is selected;  $\zeta_i\in\{0,1\}$  be a  variable indicating if data point $i\in \P \cup {\cal N}$ is misclassified. With this notation in mind, the problem of identifying the rule set that minimizes 0-1 loss becomes:

\begin{align}
\textbf{min} \quad & \sum_{i\in \cal{P}} \zeta_i +\sum_{i\in \cal{N}} \zeta_i ~~ \label{obj:01}\\
	\textbf{s.t.}  \quad
	& \zeta_i + \sum_{k\in {\cal K}_i} w_k \geq 1 , \quad i \in \cal{P}~~ \label{const:accP01} \\
	& w_k  \leq   \zeta_i , \quad  i \in {\cal N}, k \in {\cal K}_i ~~ \label{const:accZ01} \\
	& w_k,  ~\zeta_i  \in\{0,1\}, \quad {k\in \cal{K}}, i \in {\cal P} \cup {\cal N} ~~ \label{const:accBin01} 
\end{align}
Any feasible solution $(\bar w,\bar \zeta)$ to \eqref{const:accP01}-\eqref{const:accBin01} corresponds to a rule set $S=\{k\in\K\::\: \bar w_k=1\}$. 
Constraint \eqref{const:accP01} identifies false negatives by forcing $\zeta_i $ to take value 1 if no rule that is satisfied by the point $i \in \mathcal{P}$ is selected. Similarly, constraint \eqref{const:accZ01} identifies false positives by forcing $\zeta_i$ to take a value of 1 if any rule satisfied by $i \in \mathcal{N}$ is selected. The objective forces $\zeta_i$ to be 0 when possible, so there is no need to add constraints to track whether data point $i$ is classified correctly in this formulation. Note that there are a potentially exponential number of constraint (\ref{const:accZ01}) as there needs to be one constraint for each rule met by data point $i$. In the presence of a cardinality constraint on the $w_k$ (i.e. a limit on the total number of rules), the set of constraints (\ref{const:accZ01}) for each data point can be aggregated together to provide a form modern solvers will handle more efficiently \cite{IPref}.

 \subsection{Hamming loss}

While our aim is to optimize the 0-1 loss, the corresponding IP formulation is large and hard to solve in practice. Instead, we follow the approach of \cite{dash2018boolean} and optimize for {\em Hamming loss}, a proxy for classification error. For Hamming loss every data point of the positive class that is classified incorrectly incurs a cost of 1, and every data point in the negative class incurs a cost of 1 for every rule in the selected rule set that it meets.
\begin{definition}[Hamming Loss]
$$
\ell_{h}(x_i,y_i,K) = \mathbb{I}(i \in {\cal P}) \mathbb{I}(|\mathcal{K}_i \cap K| = 0) + \mathbb{I}(i \in {\cal N}) |\mathcal{K}_i \cap K|
$$
\end{definition}
Notice that Hamming loss is asymmetric with respect to errors for the positive and negative classes. Specifically, while the loss for a false negative remain the same as 0-1 loss, a false positive may incur a loss greater than 1 if it is satisfies multiple rules.
Using the same notation as the previous IP formulation, the problem of finding the rule set that minimizes Hamming loss is simply:

\begin{align}
\textbf{min} \quad & \sum_{i\in \cal{P}} \zeta_i +\sum_{i\in \cal{N}} \sum_{k \in {\cal K}_i} w_k ~~ \label{obj:ham}\\
	\textbf{s.t.}  \quad
	& \zeta_i + \sum_{k\in {\cal K}_i} w_k \geq 1 , \quad i \in \cal{P}~~ \label{const:hamP} \\
	& w_k,  ~\zeta_i  \in\{0,1\}, \quad {k\in \cal{K}}, i \in {\cal P} \cup {\cal N} ~~ \label{const:hamBin} 
\end{align}

The objective is now Hamming loss where for each $i\in \mathcal{N}$ the second term adds up the total number of selected rules satisfied by $i$. Compared to the 0-1 loss formulation, this formulation does not have the large number of constraints needed to track false positives. This leads to a much more compact formulation that is more computationally efficient to solve in practice.

However while it leads to a smaller IP formulation, we next observe that optimizing for Hamming loss can lead to arbitrarily worse performance on the 0-1 loss than optimizing for 0-1 loss directly.

\begin{theorem}[Hamming loss vs. 0-1 loss] Consider a dataset $\mathcal{D} = [(x_i, y_i)]_0^n$. Let $$K^*_{\ell} = \argmin_{K \subseteq \mathcal{K}} \sum_{i=1}^{n} \ell(x_i, y_i, K)$$ be the optimal DNF rule set $K$ for loss function $\ell$ using a set of candidate rules $\mathcal{K}$. 
There does \textbf{not} exist a global constant $M \in [1,\infty)$ such that:
$$
M \mathbb{E}_{x_i, y_i \sim \mathcal{D}}[\ell_{01}(x_i, y_i, K^*_{\ell_{01}})] \geq \mathbb{E}_{x_i, y_i \sim \mathcal{D}}[\ell_{01}(x_i, y_i, K^*_{\ell_{h}})]
$$
for all rule sets $\mathcal{K}$, data $\mathcal{D}$. In other words, when evaluating the 0-1 loss on a dataset $\mathcal{D}$, the rule set $K^*_{\ell_{h}}$ selected to minimize Hamming loss can perform arbitrarily worse than the rule set $K^*_{\ell_{01}}$ selected to minimize 0-1 loss.
\end{theorem}

\begin{proof}
Let $\mathcal{D}_n$ be a set of $n$ points sampled uniformly from the set $\mathcal{X} = \{(\textbf{x}, 1): x_1^2 +x_2^2 = 1, \textbf{x} \in \mathbb{R}_+^2 \} \}$. Consider the augmented set of points $\mathcal{D}$ which contains $\mathcal{D}_{n-2}$ together with 2 additional points at the origin with label -1, and the set of rules $\mathcal{K} = \{\mathbb{I}_{x_2 = \alpha x_1}: \alpha \in \mathbb{R}_+\}$. Graphically, the problem is represented in figure \ref{hamming_theorem_ex}. Note that the rules are simply lines intersecting the origin.
    \begin{figure*}[h]
    \centering\footnotesize
        \begin{subfigure}
          \centering
          \includegraphics[width=0.8\textwidth]{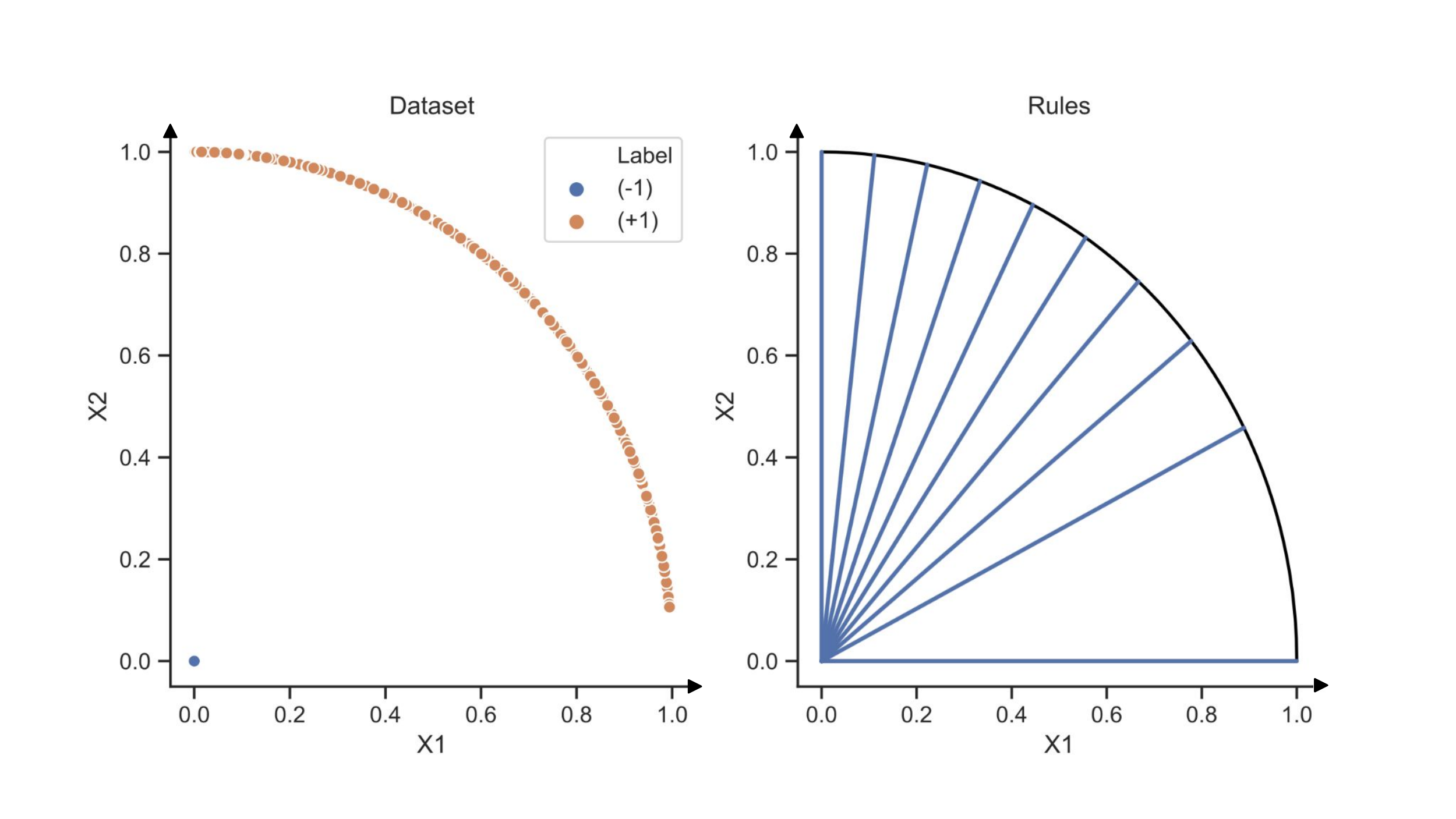}
        \end{subfigure}
         \caption {\label{hamming_theorem_ex} (Left) Sample dataset $\mathcal{D}$. Note there are two points at the origin, and $n-2$ data points in $\mathcal{D}_{n-2}$. (Right) Sample rules in $\mathcal{K}$. Note that they are all lines going through the origin. }
    \end{figure*}

Now consider the problem of constructing the optimal rule set. Note that to properly classify any data point in $\mathcal{D}_{n-2}$ correctly using rules from $\mathcal{K}$ we would need to include a rule that misclassifies the two points at the origin (as they meet every rule). Every data point in $\mathcal{D}_{n-2}$ also needs its own rule, as each rule only covers at most one point in $\mathcal{X}$. When optimizing for 0-1 loss, it is clear that the optimal rule set for $n > 4$ is to simply include the rules needed to classify all data points in $\mathcal{D}_{n-2}$ correctly which results in an expected loss of $\frac{2}{n}$. However, the asymmetry in Hamming loss means that including any rule, and by extension classifying any point in $\mathcal{D}_{n-2}$ correctly incurs a cost of 2. Thus the optimal solution when optimizing for Hamming loss is the empty rule set giving an expected 0-1 loss of $\frac{n-2}{n}$. Clearly no constant $M$ exists such that $M\frac{2}{n} \geq \frac{n-2}{n}$ for all $n$, proving the desired claim. 
\end{proof}

%While theoretically, Hamming Loss can lead to arbitrarily worse performance on the 0-1 classification problem, empirically it performs comparably (as discussed in section \ref{sec:hammingEmp}). The main reason for using Hamming Loss is practical: it leads to a smaller IP formulation which can be solved more efficiently. In other words, Hamming Loss is a computationally efficient proxy for 0-1 loss that performs well in practice, despite it's theoretical limitations.

While theoretically, Hamming loss can lead to arbitrarily worse performance on the 0-1 classification problem, we decide to use it in our formulation out of practicality. Its compact formulation can be solved more efficiently than the 0-1 formulation, leading to a more computationally tractable framework. Empirically, models trained with Hamming loss also perform comparably to those trained with 0-1 loss (as discussed in section \ref{sec:hammingEmp}). In brief, we use Hamming loss because it is a computationally efficient proxy for 0-1 loss that performs well in practice, despite its theoretical limitations.

\subsection{Base formulation without fairness considerations}
%=======================================================================

%As a reminder of notation, let $\K$ denote the set of all possible DNF rules and $\K_i\subset\mathcal{K}$ be the set of rules met by data point $i\in I$. 

In the preceding IP formulations, there is no need to ensure that $\zeta_i = 1$ only when data point $i$ is misclassified as  $\zeta_i=0$ in any optimal solution  provided that $\sum_{k \in \mathcal{K}_i} w_k=0$. However, this is not the case when additional fairness constraints are added to the formulation. To ready our formulation for the fairness setting, we add additional constraints to correctly track true positives. In our base formulation below, we also add a constraint on the total complexity of the selected rule set. Let $c_{k}$ denote the complexity of rule $k\in\K$ which is defined as a fixed cost of 1 plus the number of conditions in the rule. The constraint helps control over-fitting and also ensures the selected rule set is not too complicated and thus interpretable. Building upon the previous Hamming loss IP formulation, the problem for selecting the optimal rule set becomes:
	
%Let  $w_{k}\in\{0,1\}$  be a  variable indicating if rule $k\in\K$ is selected;  $\zeta_i\in\{0,1\}$  be a  variable indicating if data point $i\in\P$ is misclassified; and	$C$ be a parameter denoting the maximum complexity allowed. With this notation in mind, the problem of identifying the optimal rule set with Hamming Loss becomes:
	\begin{align}
	z_{mip} = \textbf{min} \sum_{i \in \mathcal{P}} \zeta_i + \sum_{i \in \mathcal{N}} \sum_{k \in \mathcal{K}_i} w_k \label{Mobj}\\
	\textbf{s.t.}\quad\quad  \zeta_i + \sum_{k \in \mathcal{K}_i} w_k &\geq 1 ~~~~~i \in \mathcal{P} \label{MmisP}\\
	C\zeta_i + \sum_{k \in \mathcal{K}_i} 2 w_k &\leq C ~~~~~i \in \mathcal{P} \label{MmisP2}\\
	\sum_{k \in \mathcal{K}} c_k w_k &\leq C \label{Mcomplex}\\
	w\in \{0,1\}^{|\mathcal{K}|},&~ \zeta\in \{0,1\}^{|\mathcal{P}|}  \label{Mbinary}
	\end{align}
%Any feasible solution $(\bar w,\bar \zeta)$ to \eqref{MmisP}-\eqref{Mbinary} corresponds to a rule set $S=\{k\in\K\::\: \bar w_k=1\}$. Note that the objective is the Hamming Loss where for each $i\in \mathcal{N}$ the second term adds up the total number of selected rules satisfied by $i$.
%Constraint \eqref{MmisP} identifies false negatives by forcing $\zeta_i $ to take value 1 if no rule that is satisfied by the point $i \in \mathcal{P}$ is selected.

Constraint \eqref{MmisP2} ensures that $\zeta_i$ can only take a value of 1 if no rules satisfied by $i \in \mathcal{P}$ are selected. Here, we use the fact that $c_k\ge2$ for all $k\in\K$ to get a tighter formulation. In principle, an even tighter formulation would replace the co-efficient of 2 with $c_k$. However, this complicates our column generation procedure, turning the pricing problem into a quadratic integer program - a much more computationally demanding problem to solve. Any approach to solve the pricing problem with a co-efficient of 2 (i.e. the formulation presented), and then substitute the true complexity into the master problem (i.e. use $c_k$ instead of 2 when solving the RMLP) runs the risk of cycling by repeatedly generating the same column. Constraint \eqref{Mcomplex} provides the bound on complexity of the final rule set.

We denote this integer program the Master Integer Program (MIP), and its associated linear relaxation the Master LP (MLP) (obtained by dropping the integrality constraint).

\subsection{Fairness constraints}\label{sec:fairconst}
%=======================================================================
To extend the IP model to incorporate the fairness criteria discussed in Section \ref{sec:fairMet} we add new constraints to the master problem.
For each group $g\in \G$ we denote the data points that have the protected feature $g$ with
	\begin{equation*}\mathcal{G}_g = \{i\in I : g_i = g \}\end{equation*}
and let $\mathcal{P}_g=\mathcal{P} \cap \mathcal{G}_g$ and $ \mathcal{N}_g=\mathcal{N} \cap \mathcal{G}_g$. For simplicity, we describe the constraints assuming $\G=\{1,2\}$ and note that extending it to multiple groups is straightforward and simply adds constraints that scale linearly with the number of groups.

\mysec{Equality of opportunity}
%=================================
To incorporate the equality of opportunity criterion, we bound the difference in the false negative rates between groups. 
	\begin{align}
	\frac{1}{\abs{\mathcal{P}_1}}\sum_{i \in  \mathcal{P}_1 } \zeta_i 
	        - \frac{1}{\abs{\mathcal{P}_2}}\sum_{i \in  \mathcal{P}_2} \zeta_i &\leq \epsilon_1 \label{const:eo1}\\
	 \frac{1}{\abs{\mathcal{P}_2}}\sum_{i \in  \mathcal{P}_2} \zeta_i 
	        - \frac{1}{\abs{\mathcal{P}_1}}\sum_{i \in  \mathcal{P}_1} \zeta_i&\leq \epsilon_1 \label{const:eo2}
	\end{align}
Constraints \eqref{const:eo1} and \eqref{const:eo2} bound the maximum allowed unfairness, denoted by $\Delta$ in the previous section, by a specified constant $\epsilon_1\ge 0$. If $\epsilon_1$ is chosen to be 0, then the fairness constraint is imposed strictly. Depending on the application, $\epsilon_1$ can also be  larger than 0, in which case a prescribed level of unfairness is tolerated.

\mysec{Equalized odds} 	
%=================================
Similar to our use of Hamming loss as a proxy for optimizing 0-1 loss for the negative class, we use it as a proxy for equalized odds. Specifically, instead of bounding the difference in false positive rates between groups we bound the difference in the Hamming loss terms for the negative class. Thus in conjunction with constraints \eqref{const:eo1} and \eqref{const:eo2}, we also include the following constraints in the formulation:
  	\begin{align}
    \frac{1}{\abs{\mathcal{N}_1}}\sum_{i \in \mathcal{N}_1} \sum_{k \in \mathcal{K}_i} w_k  - \frac{1}{\abs{\mathcal{N}_2}}\sum_{i \in \mathcal{N}_2} \sum_{k \in \mathcal{K}_i} w_k &\leq \epsilon_2\label{const:eo3}\\
     \frac{1}{\abs{\mathcal{N}_2}}\sum_{i \in \mathcal{N}_2} \sum_{k \in \mathcal{K}_i} w_k  - \frac{1}{\abs{\mathcal{N}_1}}\sum_{i \in \mathcal{N}_1} \sum_{k \in \mathcal{K}_i} w_k &\leq \epsilon_2,\label{const:eo4}
	\end{align}
where  $\epsilon_2\ge 0$ is a given constant. While constraints \eqref{const:eo1} and \eqref{const:eo2}  bound the misclassifications gap between group for the positive class $\P$, constraints \eqref{const:eo3} and \eqref{const:eo4} do the same for negative class $\Z$ using  Hamming Loss as a proxy. The tolerance parameter $\epsilon_2$ in  \eqref{const:eo3} and \eqref{const:eo4}  can be set equal to $\epsilon_1$ in  \eqref{const:eo1} and \eqref{const:eo2}, or, alternatively, $\epsilon_1$  and $\epsilon_2$  can be chosen separately. Note that we normalize the Hamming Loss terms to account for the difference in group sizes and positive response rates between groups. 

Similar to using Hamming loss in the objective, the Hamming loss proxy for false positives can lead to arbitrarily unfair classifiers with respect to the true equalized odds criterion. However once again, the Hamming loss proxy performs well empirically and generates classifiers that meet the true fairness constraint.
\mysec{Other fairness metrics} 	
%=================================
While we restrict our focus to equality of opportunity and equalized odds, we note that our framework can be adapted for any notion of classification parity (i.e. balancing false positive rates or overall accuracy). The only caveat is that for notions of fairness involving false positives our framework would use the Hamming Loss term for false positives (similar to equalized odds). We also note that under some notions of fairness (i.e. ensuring similar accuracy between the two groups) there is no guarantee that the IP will be feasible for an arbitrary set of candidate rules. In such cases, a two stage approach could be used to first generate a set of candidate clauses that are feasible for a given fairness criteria prior to optimizing for accuracy.

\subsection{Column generation framework}\label{sec:colgen}
%=======================================================================
Due to the exponential nature of our formulation, it is not  computationally tractable to enumerate all possible rules.
Consequently, it is not practical to solve the MIP \eqref{Mobj}-\eqref{Mbinary} using standard branch-and-bound techniques	\cite{LandDoig} even without the fairness constraints. To overcome this problem, we solve the LP relaxation of MIP using the column generation technique \cite{IPref, GilmoreGomory} without explicitly enumerating all possible rules. 
Once we solve the LP to optimality or near optimality, we then restrict our attention to the rules generated during the process and pick the best subset of these rules by  solving a restricted MIP. We note that it is possible to integrate the column generation technique with branch-and-bound to solve the MIP to provable optimality using the branch-and-price approach \cite{Nemhauser:1998}. However, this would be quite time consuming in practice.

To solve the LP relaxation of the MIP, called the MLP, we start with a possibly empty subset $\hat\K\subset\K$ of all possible rules and solve an LP restricted to the variables associated with these rules only. Once this small LP is solved, we use its optimal dual solution to identify a missing variable (rule) that has a negative reduced cost. The search for such a rule is called the {\em pricing problem} and in our case this can be done by solving a separate integer program.
If a rule with a negative reduced cost is found, then  $\hat \K$ is augmented with the associated rule and the LP is solved again and this process is repeated until no  such rule can be found. For large problems, solving the MLP to optimality is not always computationally feasible. Out of practicality, we put an overall time limit on the column generation process and terminate without a certificate of optimality upon reaching it.

\mysec{Equality of opportunity} 
%================================
Given a subset of rules $\hat\K\subset\K$, let the restricted MLP, defined by \eqref{Mobj}-\eqref{Mcomplex}, \eqref{const:eo1}-\eqref{const:eo2} and denoted by RMLP, be the restriction of MLP to the rules in $\hat \K$. In other words, RMLP is the restriction of MLP where all variables $w_k$ associated with $k\in\K\setminus\hat\K$ are fixed to 0.
Let $(\mu,\alpha, \lambda,\gamma^1,\gamma^2)$ be an optimal {\em dual} solution to RMLP, where variables $\mu,\alpha,\lambda\ge0$ are associated with constraints \eqref{MmisP}, \eqref{MmisP2}, and \eqref{Mcomplex}, respectively. Variables $\gamma^1$ and $\gamma^2$ are associated with fairness constraints \eqref{const:eo1} and  \eqref{const:eo2}.
Using this dual solution, the reduced cost of a variable $w_k$ associated with a rule $k\notin\hat\K$ can be expressed as
\begin{equation}
\hat \rho_k =  \sum_{i \in \mathcal{N}} \mathbbm1_{\{k\in\K_i\}}
-\sum_{i \in \mathcal{P}}  \mu_i \mathbbm1_{\{k\in\K_i\}} 
+ \sum_{i \in \mathcal{P}} 2 \alpha_i  \mathbbm1_{\{k\in\K_i\}} +\lambda  c_k  \label{eq:redcost}
\end{equation}
where the first term simply counts the number of data points $i\in\Z$ that satisfy the rule $k$.
Note that variable $w_k$ does not appear in constraints \eqref{const:eo1} or  \eqref{const:eo2}  in  RMLP and consequently \eqref{eq:redcost} does not involve variables $\gamma^1$ or $\gamma^2$.
If there exists a $k\in\K\setminus\hat\K$ with $\hat \rho_k<0,$ then including variable $w_k$ in RMLP has the potential of decreasing the objective function. Also note that $\hat \rho_k\geq0$ for all $k\in\hat\K$ as the dual solution at hand is optimal.

We can now formulate an integer program to find a $k\in\K$ with the minimum reduced cost $\hat \rho_k$.  
Remember that  a decision rule corresponds to a subset of the binary features $J$ and classifies a data point with a positive response if the point has all the features selected by the rule. 
Let variable $z_j\in\{0,1\}$ for  $j\in J$  denote if the rule has feature $j$ and let variable  $\delta_i\in\{0,1\}$ for $i\in I$ denote if the rule misclassifies data point $i$.
Using these variables, the complexity of a rule can be computed as $(1 + \sum_{j \in J} z_j)$ and the reduced cost of the rule becomes:

    \begin{equation} \label{eq:pricingObj}
     \enskip \sum_{i \in \mathcal{N}} \delta_i + \sum_{i \in \mathcal{P}} (2\alpha_i - \mu_i ) \delta_i + 
     \lambda(1 + \sum_{j \in J} z_j)  .
    \end{equation}

%where $\delta_i$ is the binary variable denoting whether the rule misclassifies sample $i$, $\lambda$ is the dual value associated with constraint (3), and $\mu$ are the dual values associated with constraint (2) respectively. With that objective in mind we also need to add constraints to the pricing problem to include limits on rule complexity and proper tracking of the misclassifcation of each data point. 

The full pricing problem thus simply minimizes (\ref{eq:pricingObj}) subject to the constraints:
    \begin{align}
	D\delta_i + \sum_{j \in S_i} z_j &\leq D ~~~~~ i \in I^{-}\quad \label{PmisP}\\
	\delta_i +  \sum_{j \in S_i} z_j &\geq  1~~~~~i \in I^{+}  \quad\label{PmisZ}\\
	\sum_{j \in J} z_j &\leq D \quad \label{Pcomplex}\\
	z\in \{0,1\}^{|J|},&~ \delta\in \{0,1\}^{|\mathcal{P}|}\label{Pint}\qquad
	\end{align}
where the set $I^-\subseteq I$ contains the indices of $\delta_i$ variables that have a negative  coefficient in the objective (i.e. $2\alpha_i - \mu_i < 0$), and $I^+=I\setminus I^-$. Note that constraints \eqref{PmisP} and \eqref{PmisZ} ensure that $\delta_i$ accurately reflects whether the new rule classifies data point $i$ with a positive label, and constraint \eqref{Pcomplex} puts an explicit bound on the complexity of any rule using the parameter $D$. This individual rule complexity constraint can be set independently of $C$ in the master problem or can simply be set to $C - 1$.  

We note that  in \cite{dash2018boolean} a similar pricing problem was formulated using the constraints $\delta_i + z_j\leq 1$  for all $i \in \mathcal{P},j \in S_i$, in place of \eqref{PmisP}.
The formulation we use here is computationally more efficient as it has much fewer constraints.

%%%%%%%%%%%%%%%%%%%%%%%%%%%%%%%%%%%%%%%%%%%%%

\mysec{Equalized odds} 	
In this case the RMLP is defined by \eqref{Mobj}-\eqref{Mcomplex}, \eqref{const:eo1}-\eqref{const:eo4} and note that unlike  \eqref{const:eo1} and  \eqref{const:eo2}, constraints \eqref{const:eo3} and  \eqref{const:eo4} do involve variables $w_k$.
Let $(\mu,\alpha,\lambda,\gamma^1,\gamma^2,\gamma^3,\gamma^4)$ be an optimal {\ dual} solution to RMLP, where variables  $\gamma^3$ and $\gamma^4$ are associated with fairness constraints \eqref{const:eo3} and  \eqref{const:eo4}, respectively.
Using this dual solution, the reduced cost of a variable $w_k$ associated with  $k\notin\hat\K$ is similar to the expression in \eqref{eq:redcost}, except it has the following 4 additional terms:

\begin{align*}
%\hat \rho_k = \lambda  c_k    - \sum_{i \in \mathcal{P}} \mu_i\mathbbm1_{\{k\in\K_i\}}  
%+ \sum_{i \in \mathcal{N}} \mathbbm1_{\{k\in\K_i\}}
  \sum_{i \in \mathcal{N}_1}  \frac{\gamma_3}{\abs{\mathcal{N}_1}} \mathbbm1_{\{k\in\K_i\}}
- \sum_{i \in \mathcal{N}_1} \ \frac{\gamma_4}{\abs{\mathcal{N}_1}}\mathbbm1_{\{k\in\K_i\}}
- \sum_{i \in \mathcal{N}_2}  \frac{\gamma_3}{\abs{\mathcal{N}_2}}\mathbbm1_{\{k\in\K_i\}}
+ \sum_{i \in \mathcal{N}_2} \frac{\gamma_4}{\abs{\mathcal{N}_2}}\mathbbm1_{\{k\in\K_i\}}
\label{eq:redcostH}
\end{align*}

% to save a couple of lines we can uncomment next and comment out: \\	&\textbf{s.t.} ~~\eqref{PmisP}-\eqref{Pint}.
Consequently, the %objective of the 
pricing problem becomes
    \begin{align*}
	~&\textbf{min} ~~  
	 (1+\frac{\gamma_3 - \gamma_4}{\abs{\mathcal{N}_1}})\sum_{i \in \mathcal{N}_1} \delta_i 
	+ (1+\frac{\gamma_4 - \gamma_3}{\abs{\mathcal{N}_2}})\sum_{i \in \mathcal{N}_2} \delta_i  
	&+ \sum_{i \in \mathcal{P}} (2\alpha_i - \mu_i) \delta_i 
	+ \lambda(1 + \sum_{j \in J} z_j) 
	\\	&\textbf{s.t.} ~~\eqref{PmisP}-\eqref{Pint}.
	\end{align*}

\mysec{Upper Bounds on the variables in the LP Relaxation} 
%Need help fitting this in somewhere (also the tone to talk about it)
When solving the LP relaxation of the IP model we relax not just the integrality of the decision variables, but also the upper bound of 1 on binary variables. This is because modern optimization solvers implicitly add the upper bound on the binary decision variables, causing an additional term in the formula for the reduced cost that is not accounted for in our model. Take a simple example with one data point with a single binary feature ($X = [1]$) and a positive response ($Y = 1$). Solving the master LP problem returns the following dual values:

%Also need some latex magic to get this to stay between paragraphs
\begin{table}[h]
    \centering
    \begin{tabular}{c|c|c|c}
         Constraint & (4)& (5)& (6) \\ \midrule
         Dual Value & $\mu =1$& $\alpha = 0$& $ \lambda =0$  
    \end{tabular}
    \caption{Dual Values}
    \label{tab:my_label}
\end{table}

The pricing problem predictably returns the optimal rule consisting of the single feature. However solving the master LP with the optimal rule set gives identical dual values - prompting the column generation process to cycle. However, if we explicitly add the upper bound on the binary variable (i.e. $w_k \leq 1$), we see that it has an associated dual value of $-1$ giving the optimal rule a reduced cost of 0 in the pricing problem (and thus a certificate that the LP relaxation has been solved). However, such a dual variable cannot be integrated into the pricing problem formulation as each rule would have its own associated constraint and thus the reduced cost term cannot be integrated into the objective of the pricing problem. This is not just a feature of the toy example, this cycling phenomenon occurred during initial experiments with our formulation. To get around this issue, we remove the constraint on $w_k$ being binary (i.e. we allow a rule to be included multiple times in a rule set). During our experiments, we did not observe any instances where $w_k > 0$ for a solution.

\section{Computational approach} \label{sec:comp}

    We use the Python interface of Gurobi \cite{gurobi} to solve the linear and integer programs in our formulation. To solve the MLP we use a barrier interior point method, which performs better in practice than simplex for large sparse problems with many columns. For the pricing problem we use the default settings and return all solutions generated during the algorithm's run with negative reduced costs. In many problems, solving either the integer pricing problem or the MLP to optimality can be computationally intensive. To place a practical limit on the overall problem, we set time limits on the pricing problem and the overall column generation (CG) process to solve the MLP.

    During the solving process, the MIP solver retains a set of feasible solutions generated for the master problem. We evaluate the 0-1 loss on the training set for all solutions generated, and return the rule set with the lowest 0-1 loss rather than the optimal solution for the hamming loss problem. In other words, while we optimize for hamming loss, we select the best solution from the candidate pool of feasible solutions using 0-1 loss. This comes at a negligible computational cost (i.e. simply the cost of doing inference on the training set for each solution), and can lead to a small bump in performance with respect to the 0-1 problem. Table \ref{tab:0-1select} summarizes the impact of selecting a final rule set from the solution pool using hamming loss and accuracy respectively. Selecting the final rule set using accuracy leads to a modest increase in both train and test set accuracy across all three datasets. Note that while the bump in performance is modest, it comes at practically no additional computational cost. 
    
       \begin{table*}[h]
\centering\footnotesize
\caption{\label{tab:0-1select} Effect of selecting a final rule set from the solution pool using hamming loss and accuracy (standard deviation in parenthesis)}
\setlength{\tabcolsep}{5pt} % Default value: 6pt
\begin{tabular}{l  c c c c c c }		\toprule
& \multicolumn{2}{c}{Adult} & \multicolumn{2}{c}{Compas} & \multicolumn{2}{c}{Default} \\
 &Hamming & Accuracy &Hamming & Accuracy & Hamming & Accuracy\\\midrule
Hamming Loss & 5305 (64) & 5568 (441) & 1624 (54)  & 1715 (306) & 5950 (19) & 5956 (15)\\
Train Accuracy & 81.9 (0.1) & 82.1 (0.2) & 66.3 (0.3) & 66.5 (0.3) & 77.9 (0.003) &77.9 (0.03) \\
Test Accuracy & 81.6 (0.7) & 81.8 (0.8) & 66.2 (2.0) & 66.3 (2.7) &78.1 (0.02)&78.1 (0.5)\\
\bottomrule
\end{tabular}%
\end{table*}%

    For small problems, defined roughly as having under 2000 data points and a couple hundred binary features, we  employ the  CG framework described in Section \ref{sec:colgen}, however for larger problems we use an approximate version of the framework to limit the overall run time. For large datasets, the integer programming formulation for the pricing problem may become computationally intractable and therefore we sub-sample both our training data points and potential features to have on average 2000 rows and 100000 non-zeros, we then solve the pricing problem on the reduced problem. 
    
    To warm start our column generation process, we start by mining a candidate set of rules from quick tree-based heuristics similar to \cite{birbil2020rule} (see Section \ref{sec:ruleMineEmp} for a discussion on how warm starting saves computationally intensive CG iterations). 
    %In addition to solving the IP formulation of the pricing problem, we also use a greedy heuristic (similar to  \cite{dash2018boolean}) for generating rules with up to five features. We start with all rules with a single feature and compute their reduced costs using the dual solution of the RMLP. We keep the best 20 rules and then iteratively build up rule size by testing all k-feature rules built off of the 20 retained (k-1)-feature rules. All rules with negative reduced costs (even those not in the best 20 for each rule size) are returned. In practice, we found employing the heuristic for a fixed number of CG cycles, 3 in our experiments, and then switching to the IP formulation to 'fine-tune' the rule-set produced the best results. 
    If a large number of candidate rules are generated (for both the heuristic and integer programming formulation) we return the 100 rules with the lowest reduced costs. This is to trade off the number of column generation iterations needed with the size of the restricted LP. Adding all columns leads to larger candidate rule sets that slow down solving the RLP, while only adding the best new rule tends to lead to more iterations of column generation needed.

	\section{Experimental results} \label{sec:exp}
	%=======================================================================
  
    To benchmark the performance of our algorithm, we evaluate it on three standard fair machine learning datasets: default \cite{UCI}, adult \cite{UCI}, and compas \cite{propub}. For our experiments we include the sensitive attribute as a feature, though our framework can easily work with the sensitive attribute exclude for inference. We refer to \cite{corbettdavies2018measure} as to why including the sensitive attribute can lead to fairer results, absent regulatory constraints on disparate treatment. Details on the sources of each dataset and how we pre-process data can be found in the appendix.

   To convert the data to be binary-valued, we use a standard methodology also used in \cite{dash2018boolean, su2016, wang2017}. 
    For categorical variables $j$ we use one-hot encoding to binarize each variable into multiple indicator variables that check $X_{j} = x$, and the negation $X_{j} \neq x$.
    For numerical variables we compare values against a sequence of thresholds for that column and include both the comparison and it's negation (i.e. $X_j \leq 1, X_j \leq 2$ and $X_j > 1, X_j > 2$). For our experiments we use the sample deciles as the thresholds for each feature. We use the binarized data for all the algorithms we test to control for the binarization method.

    For all our results we used ten-fold cross validation and performed a two-stage process to generate a pool of candidate rules and build a rule set. First, we run the CG algorithm on the training data for a subset of potential hyperparameters (typically 3 different complexity limits and epsilon values). Then, we use these candidate rules and solve the master IP for a grid of potential epsilon values and complexity limits. To select which complexity values to test we look at which complexity value leads to the best cross-validated accuracy in the problem without fairness constraints and test values around and including it. We use a 45 second time limit for each iteration of the pricing problem, and a five minute overall time limit for the column generation process.

    \subsection{Hamming loss}\label{sec:hammingEmp}
    %=======================================================================
    To analyze the empirical performance of using Hamming Loss instead of 0-1 loss we ran a sequence of experiments where we evaluated the 0-1 loss of rule sets trained under both objectives. For each experiment we used the same pool of candidate rules generated by running our column generation procedure with different hyperparameters. We then ran the master model under both objectives with a fixed fairness threshold ($\epsilon$) and evaluated the rule set's performance on both training and testing data. 10-fold cross-validation was used to tune hyperparameters for each model separately.
    
    Figure \ref{hamming_v_01} shows the results for the compas dataset \cite{propub} with $\epsilon = 0.1$. The left side of the figure is a violin plot that shows the distribution of accuracy results over 1000 random splits of the dataset when optimizing for each objective (each black dot represents an accuracy for a single split). We can see that Hamming Loss and Accuracy have practically indistinguishable performance in terms of both train and test set accuracy. However, the Hamming Loss model solves the IP problem in a fraction of the time of the model minimizing accuracy as seen in the plot on the right. This shows that empirically, Hamming Loss is an effective proxy for accuracy - leading to comparable performance in a fraction of the computation time. Results for other datasets and values for $\epsilon$ can be found in the appendix and show a similar trend.

    \begin{figure}[!htb]
    \centering
    \begin{subfigure}
      \centering
      \includegraphics[width=0.95\textwidth]{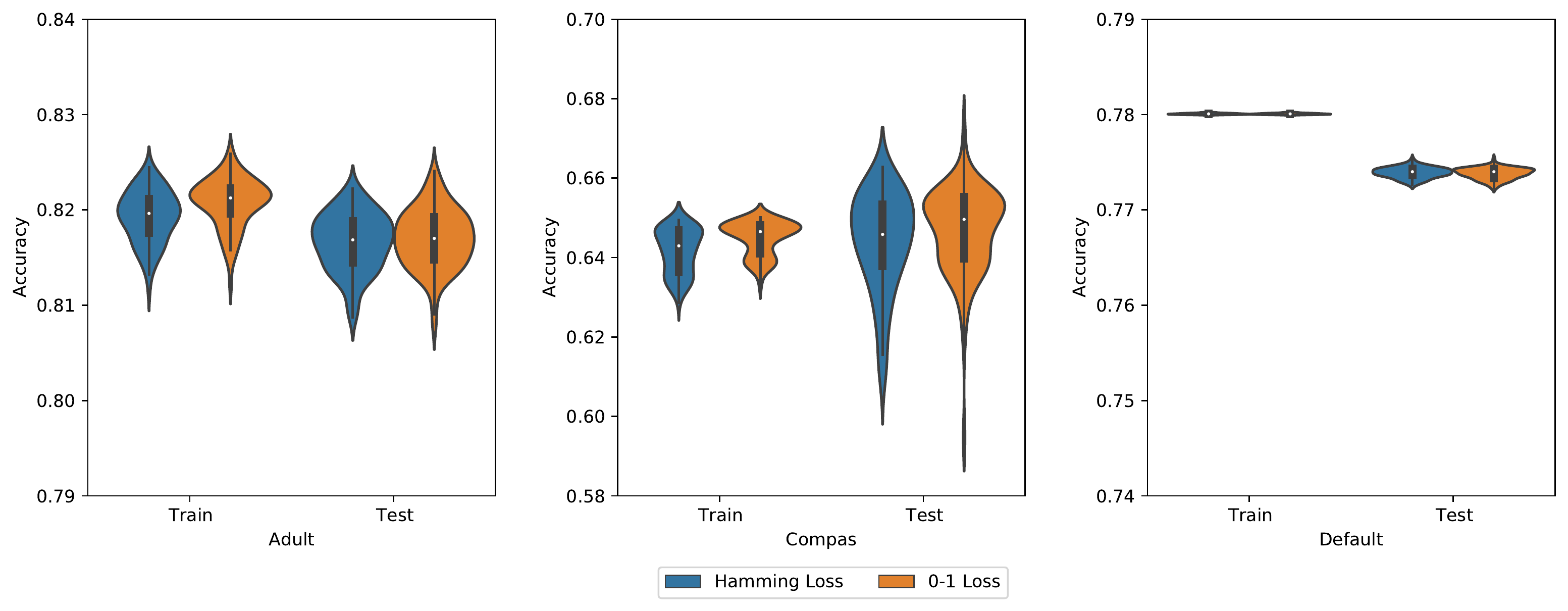}
    \end{subfigure}
  \caption {\label{hamming_v_01} Relative performance of models trained to maximize accuracy or minimize Hamming Loss respectively under equality of opportunity constraint ($\epsilon = 0.025$). The violin plots show the distribution of train and test set accuracy for models under each objective for 100 random splits of the datasets.}
\end{figure}

 \begin{table*}[h]
\centering\footnotesize
\caption{\label{hammingComputationTable} Mean (standard deviation) computation time in seconds for master model under different objectives with 600 s time limit}
\setlength{\tabcolsep}{5pt} % Default value: 6pt
 \begin{adjustbox}{center}
\begin{tabular}{l c c c c c c c c c}		\toprule
& \multicolumn{3}{c}{Adult} & \multicolumn{3}{c}{Compas} & \multicolumn{3}{c}{Default} \\
$\epsilon$&0.025&0.1&0.5&0.025&0.1&0.5&0.025&0.1&0.5\\\midrule
Hamming Loss& 35.5 (48.6) & 34 (24.1) & 30.1 (27.1) & 4.0 (3.5) & 3.5 (3.2) & 0.49 (0.3) & 3.5 (0.8) & 3.4 (0.7) & 3.7 (0.7) \\\midrule
0-1 Loss & 546 (101) & 594 (33) & 520 (78) & 11.4 (6.5) & 12.1 (6.5) & 6.0 (3.6) & 12.3 (6.1) & 12.8 (7.2) & 12.5 (4.7) \\
\bottomrule
\end{tabular}%
\end{adjustbox}
\end{table*}%

In addition to using Hamming Loss as a proxy for 0-1 error, we also use Hamming Loss terms to bound the false negative error rate in the equalized odds formulation (constraints (\ref{const:eo3}) and (\ref{const:eo4})) A natural question is how well the Hamming Loss terms work as a proxy for false negatives in equalized odds (i.e. whether the constraint does in fact bound allowable unfairness under equalized odds). Figure \ref{hamming_fair} plots the epsilon used to constrain the Hamming Loss proxy for equalized odds on the x-axis, versus the true equalized odds of the resulting rule set both on the training and testing data (averaged over 10 folds) for the three datasets. We can see that in practice, the Hamming Loss version of the false negative constraint serves as an effective proxy for the true fairness constraint. Empirically, the rule sets trained with the Hamming Loss constraint never have a true equalized odds larger than the prescribed $\epsilon$ for the training data. The out of sample equalized odds does exceed the constraint in some scenarios, as expected due to it being new data not used to train the model, but tracks relatively close to the true equalized odds on the training data. This is promising evidence that the Hamming Loss terms are both an effective proxy for the true fairness constraint and that the fairness guarantees on the training set generalize well to unseen data.

    \begin{figure}[!htb]
        \centering
        \begin{subfigure}
          \centering
          \includegraphics[width=0.32\textwidth]{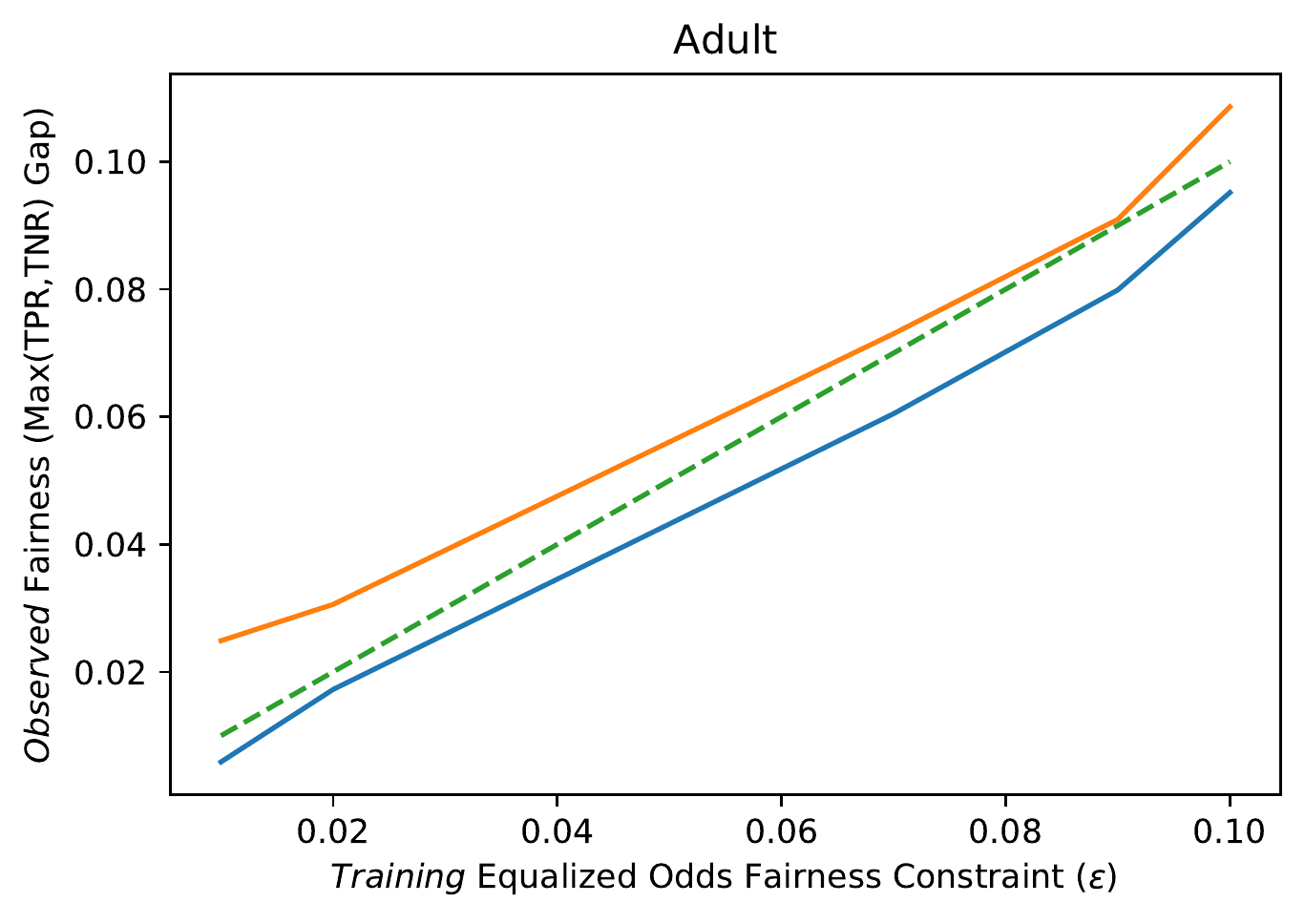}
        \end{subfigure}
        \begin{subfigure}
          \centering
          \includegraphics[width=0.32\textwidth]{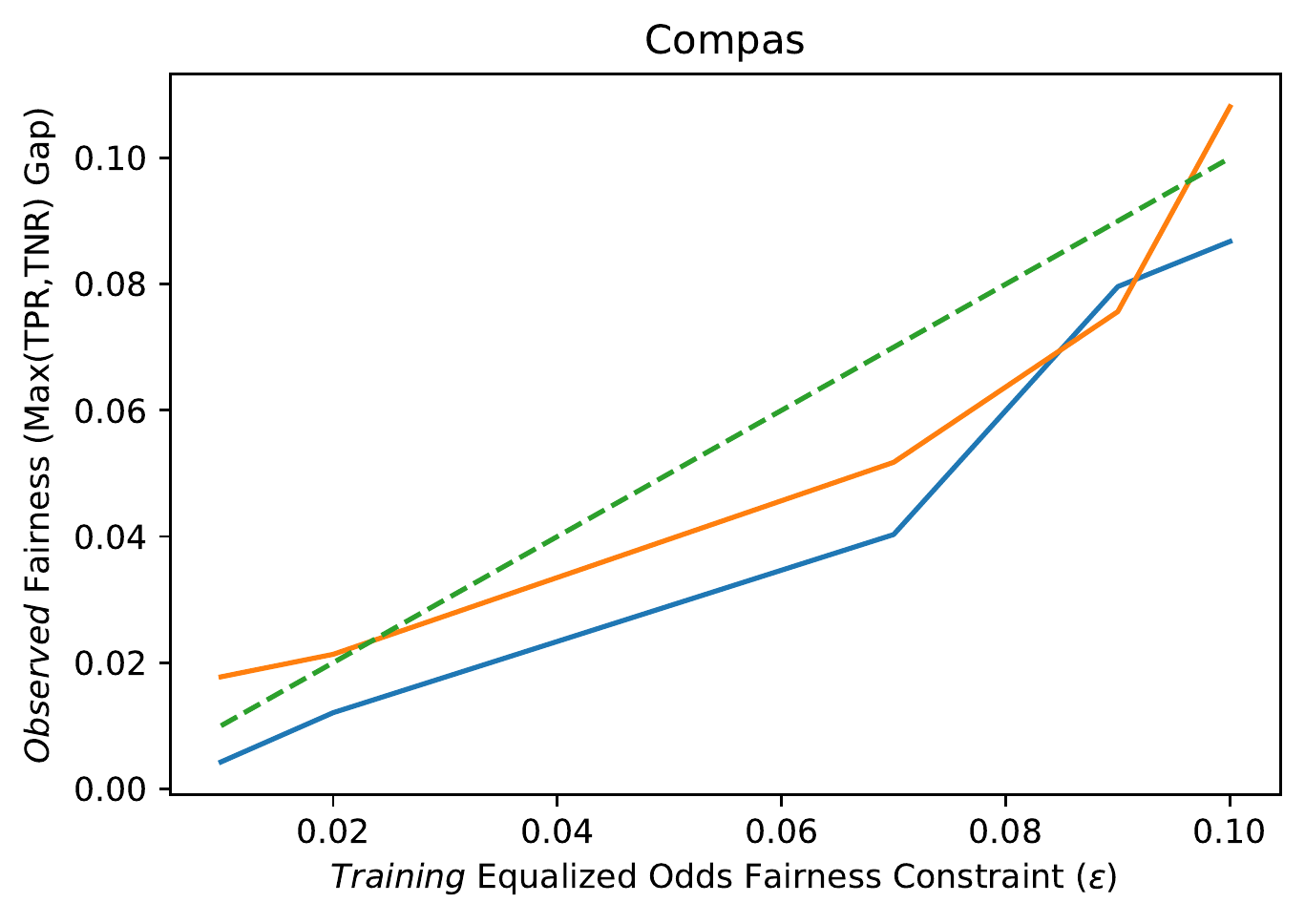}
        \end{subfigure}
        \begin{subfigure}
          \centering
          \includegraphics[width=0.32\textwidth]{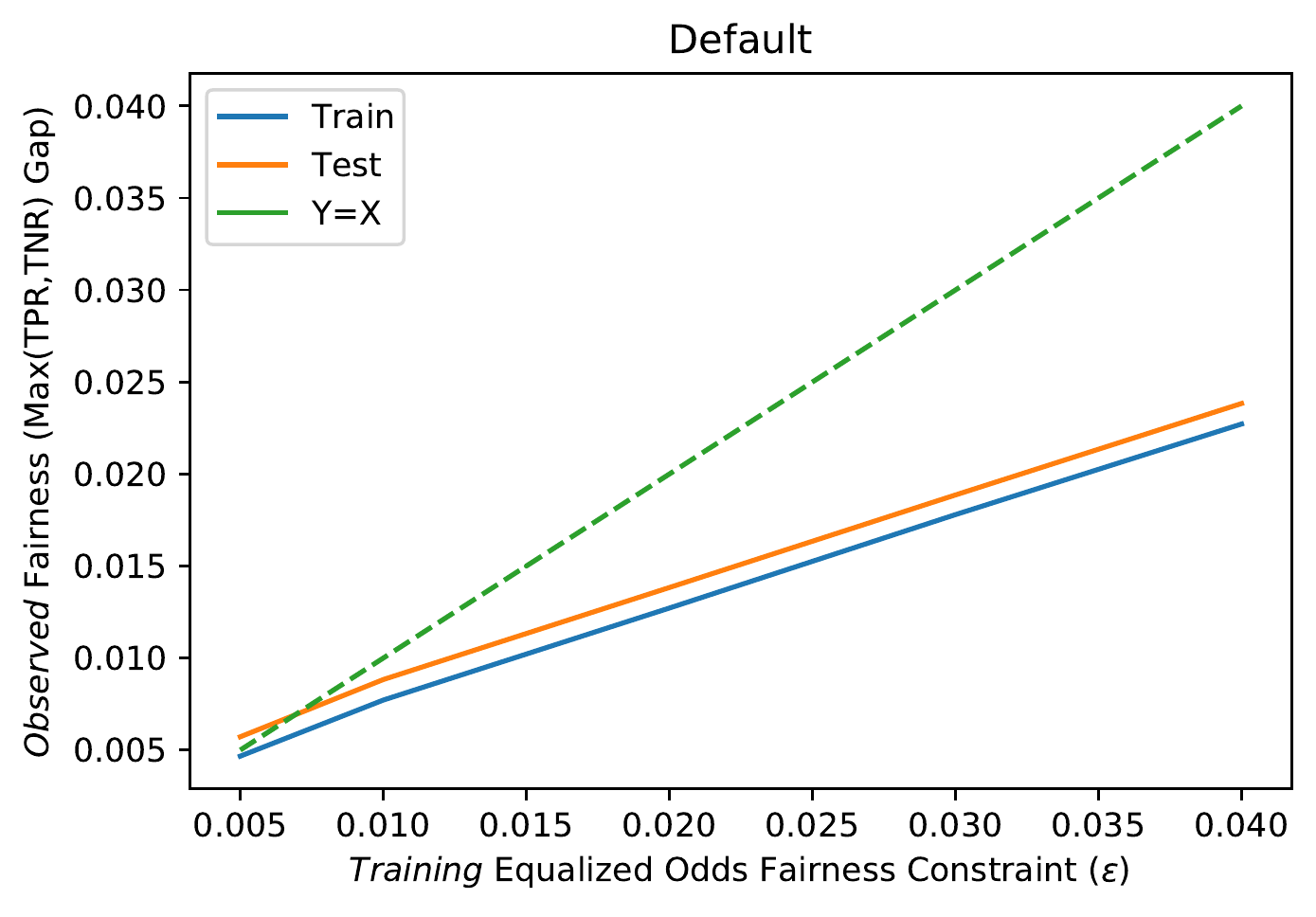}
        \end{subfigure}
         \caption {\label{hamming_fair} Generalization of equalized odds constraint with Hamming loss proxy for false negative rate versus true equalized odds gap measure on both training and testing data.  }
    \end{figure}

    \subsection{Rule mining}\label{sec:ruleMineEmp}
    %=======================================================================
    As we do not perform branch and price, our column generation procedure can be viewed as a heuristic for generating a set of candidate rules to use in the master IP model. To that end, we evaluated how our method compares to other rule mining heuristics. Following the procedure outlined in \cite{birbil2020rule}, we mined rules from a variety of tree based models including CART, a random forest classifier, and a fair decision tree learned through the fair learn package \cite{agarwal2018reductions}. We used tree-based models to mine rules because they have a natural correspondence to rule sets, as every path from the root of the tree to any leaf node can be seen as a rule. We extract all such rules where the label of the associated leaf node is for the positive class.  For each tree-based model, we trained the algorithm with a variety of hyperparameters and extract every rule generated by the classifier. We then run the master IP model with the extracted rules for a fixed fairness criteria ($\epsilon$) and do cross-validation to select the best complexity bound with each rule set. 
    
    Figure \ref{rule_mining} shows the relative performance of each rule set on the compas dataset for $\epsilon = 0.01$ with the equality of opportunity metric. FairCG produced the rule sets that lead to the lowest objective value when compared to any other rule sources. However, there is a modest increase in performance by looking at the union of all the rule sets. While they do not lead to superior performance, the rules mined from quicker methods like random forests provide a powerful warm start to the column generation procedure to solve the MIP. The second plot in figure \ref{rule_mining} shows the number of column generation iterations needed to find a collection of rules with comparable accuracy to the mined rules from a random forest model. The results show that using a pre-mined rule set can reduce the number of time intensive iterations of column generation needed. Results for additional datasets and $\epsilon$ can be found in the appendix.

    \begin{figure}[!htb]
        \centering
        \begin{subfigure}
          \centering
          \includegraphics[width=0.47\textwidth]{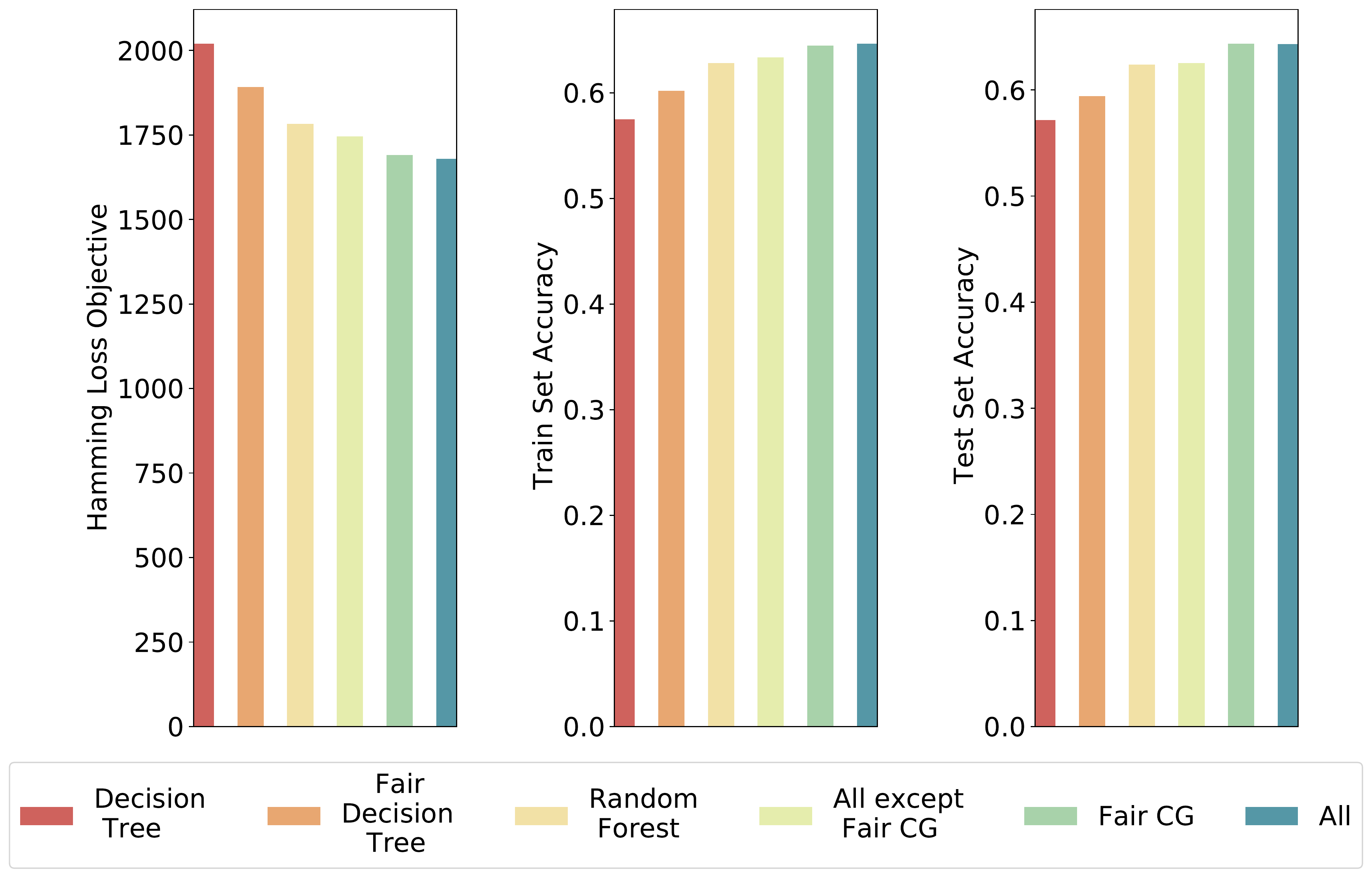}
        \end{subfigure}
        \begin{subfigure}
          \centering
          \includegraphics[width=0.47\textwidth]{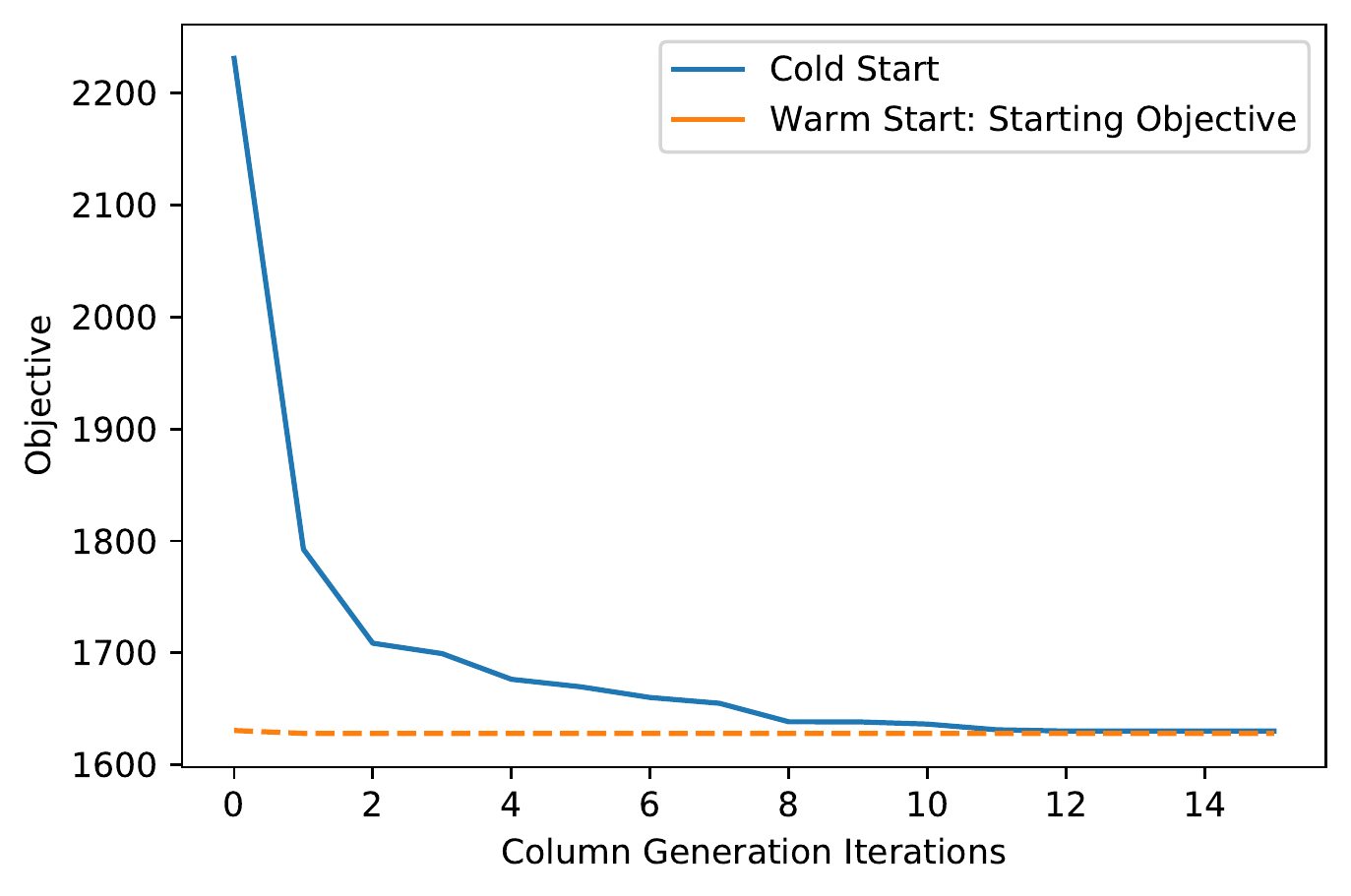}
        \end{subfigure}
         \caption {\label{rule_mining} (Left) Relative performance of rules mined from different sources ('All' signifies union of all the rule sets). (Right) Number of column generation iterations (blue line) needed to achieve comparable performance to rules mined from random forest model (dotted line). }
    \end{figure}

    \subsection{Fairness - accuracy trade-off}\label{sec:fairext}
    %=======================================================================
    
    We compared our performance against three other popular interpretable fair binary classification models: Zafar 2017 \cite{Zafar_2017}, Hardt 2016 \cite{hardt2016equality}, and the exponential gradient method included in the Fair-Learn package \cite{agarwal2018reductions}. The first two methods build fair logistic regression classifiers but take different approaches. Zafar 2017 formulates the problem as a constrained optimization model and uses a convex relaxation of the fairness constraint. Hardt 2016 also leverages a logistic regression classifier but achieves fairness by selecting different discrimination thresholds for the sensitive groups. The exponential gradient method from Fair-Learn works by solving a sequence of cost-sensitive classification problems to construct a randomized classifier with low error and the desired fairness. The framework works with any classifier, however for our experiments we chose to use a decision tree as the base learner. 
    
    We started by comparing the algorithm's predictive accuracy in the absence of fairness criteria. Table \ref{accTableeNoConst} shows the 10-fold mean and standard deviation accuracy for each algorithm without fairness criteria (i.e. $\epsilon = 1$). While the algorithms are not trained to consider fairness, we report the average test set 'unfairness' of each algorithm as a baseline for the amount of discrimination that happened in the absence of controls on fairness. On two of the three datasets (compas and default), rule sets have the strongest predictive performance. However, on the adult dataset the rule sets were outperformed by the two logistic regression based classifiers (Zafar and Hardt).

\begin{table*}[h]
\centering
\caption{\label{accTableeNoConst} Mean accuracy and fairness results with no fairness constraints (standard deviation in parenthesis). Equality of opportunity and equalized odds refer to the amount of unfairness between the two groups under each fairness metric.}
\setlength{\tabcolsep}{5pt} % Default value: 6pt
\begin{tabular}{l l  c c c c }		\toprule
& & Fair CG & Zafar 2017 & Hardt 2016 & Fair Learn \\\midrule
\multirow{3}{*}{Adult} & Accuracy & 82.5 (0.5) & 85.2 (0.5) & 83.0 (0.4) & 82.4 (0.4) \\
& Equality of Opportunity & 7.6 (0.5) & 11.9 (3.7) & 18.2 (4.8) & 11.5 (4.6) \\
& Equalized Odds & 7.6 (0.5) & 11.9 (3.7) &  18.2 (4.8) & 11.5 (4.6) \\ \midrule
\multirow{3}{*}{Compas} & Accuracy & 67.6 (1.1) & 64.6 (1.9) & 65.9 (2.7) & 65.8 (2.9)\\
& Equality of Opportunity & 23.8 (5.3)& 42.8 (5.4) & 23.7 (6.4) & 21.7 (7.1) \\
& Equalized Odds &24.1 (5.1) & 47.6 (5.8)& 27.0 (5.2) & 24.9 (4.5) \\ \midrule
\multirow{3}{*}{Default} & Accuracy & 82.0 (0.7) & 81.2 (0.8) & 77.9 (1.7) & 77.9 (1.7) \\
& Equality of Opportunity & 1.3 (0.6) & 2.7 (1.9) & 0 (0) & 0 (0) \\
& Equalized Odds & 1.9 (0.5) & 4.2 (2.5) & 0 (0) & 0 (0) \\ 
\bottomrule
\end{tabular}%
\end{table*}%

   We now consider adding constraints on the allowable unfairness. For each model we varied the hyperparameters in all the algorithms, performing 10-fold cross validation for each hyperparameter, to generate the accuracy fairness trade-offs.  Figure \ref{benchmark_combined} plots the fairness accuracy trade-offs under both notions of fairness. Fair CG generated classifiers that performed well, dominating all other fair classifiers on two of the three datasets. Our algorithm performs especially well when generating classifiers under strict fairness. Specifically, we dominate all other algorithms in regimes where unfairness is restricted to less than 2.5\% with either fairness criterion. 
   %Tables \ref{eqOptable} and \ref{eqOdtable} summarize each algorithm's performance when the hyper parameters are selected to minimize unfairness (excluding trivial classifiers that just assign the majority class). In both tables, $-$ indicates that no non-trivial classifier was found. In all datasets and fairness metrics, with the exception of equalized odds for the adult data set, FairCG is able to find the fair-est non-trivial classifier. 
   Overall these results show that our framework is able to build interpretable models that have competitive accuracy and substantially improved fairness.   Moreover, our algorithm allows for especially fine control over unfairness. Figure \ref{diverging_tpr} shows the effect of relaxing the fairness constraint and its associated affect on the true positive rate of both groups for the compas dataset.
   We emphasise that the allowed unfairness level during training translates directly to (practically) the same observed unfairness level in testing, thus establishing the robustness of our approach. For the remainder of our results, and more specifics on our experimental framework we refer you to the the appendix.

    \begin{figure*}[t]
    \centering
        \begin{subfigure}
          \centering
          \includegraphics[width=0.4\linewidth]{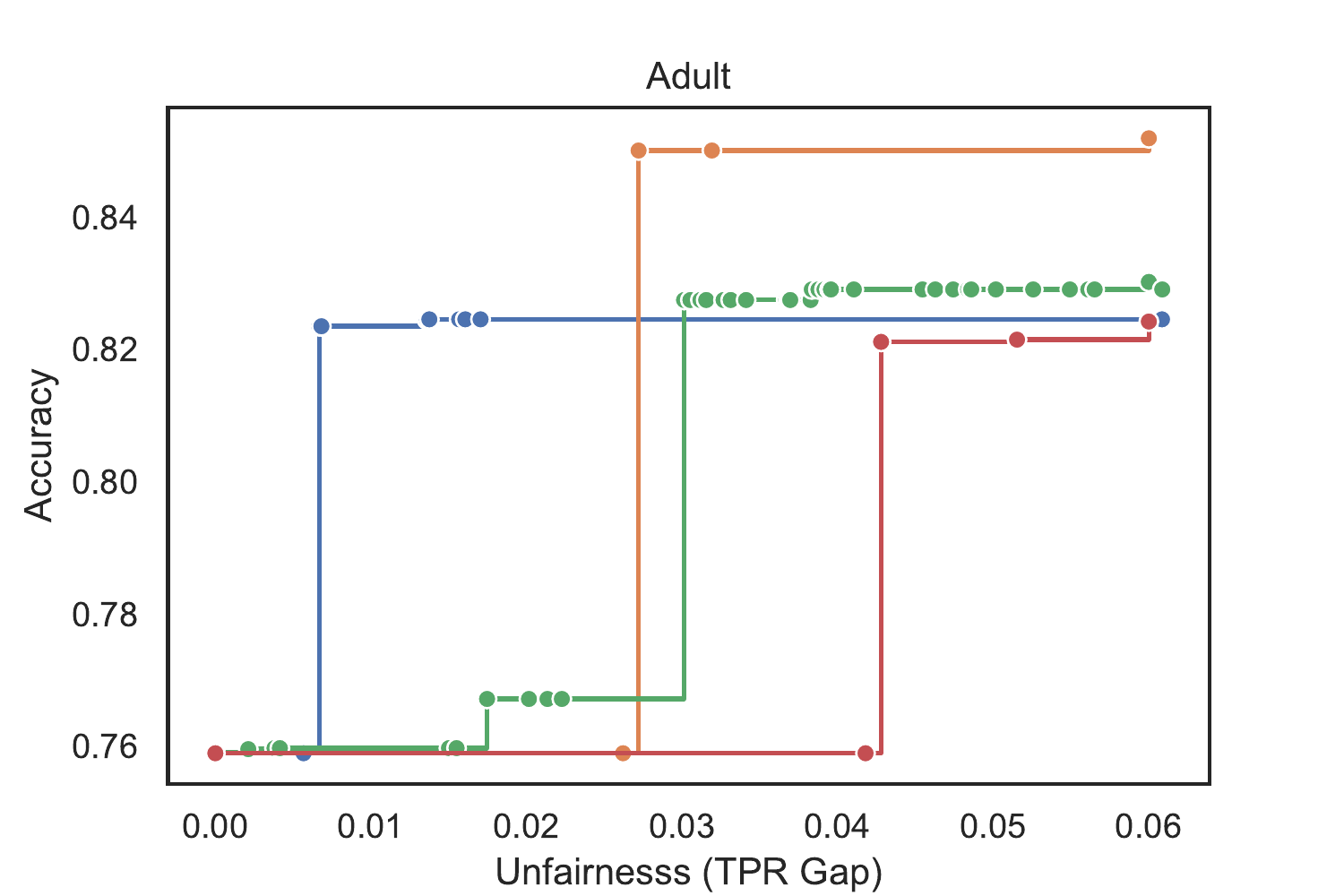}
        \end{subfigure}
        \begin{subfigure}
          \centering
          \includegraphics[width=0.4\linewidth]{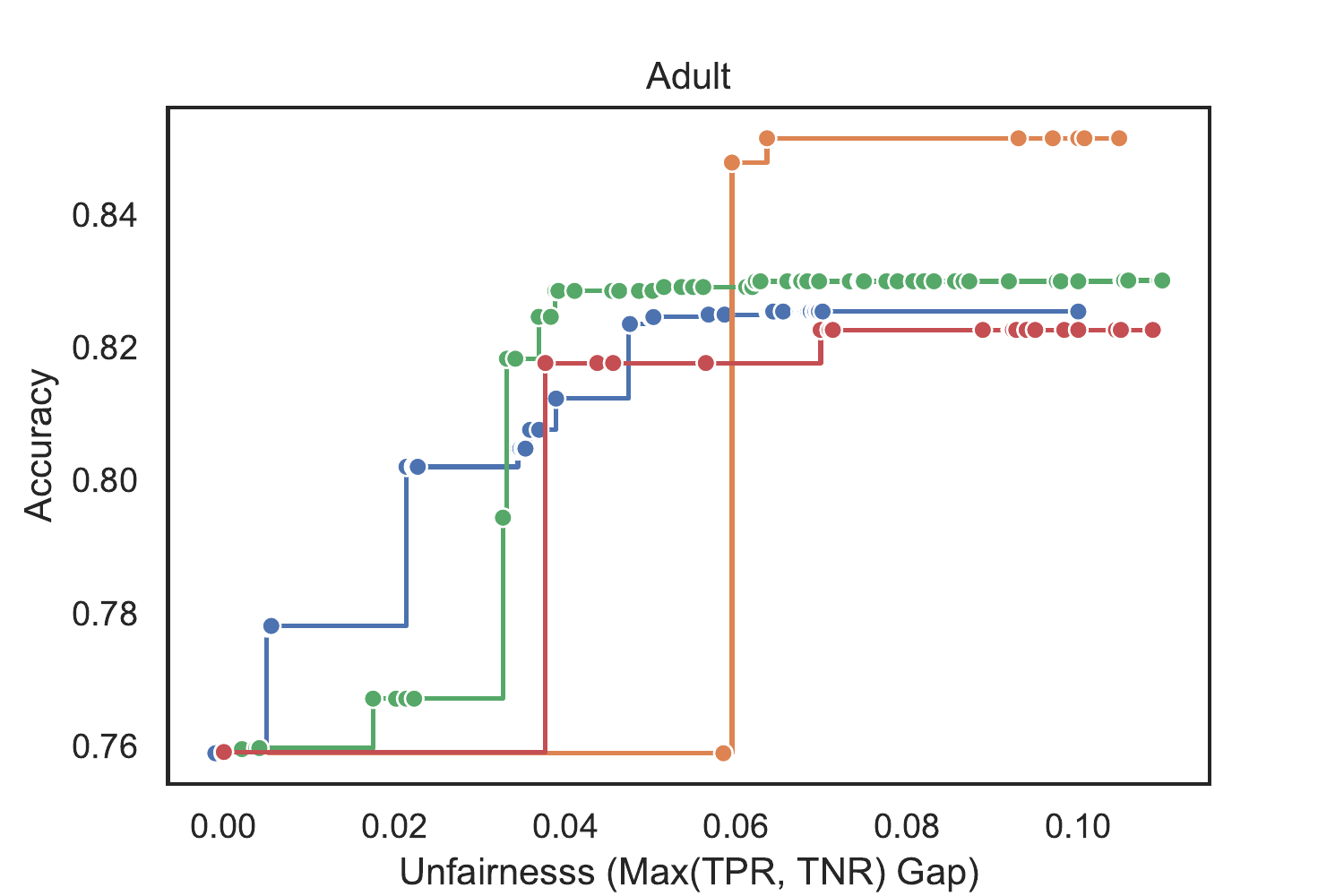}
        \end{subfigure}
        \begin{subfigure}
          \centering
          \includegraphics[width=0.4\linewidth]{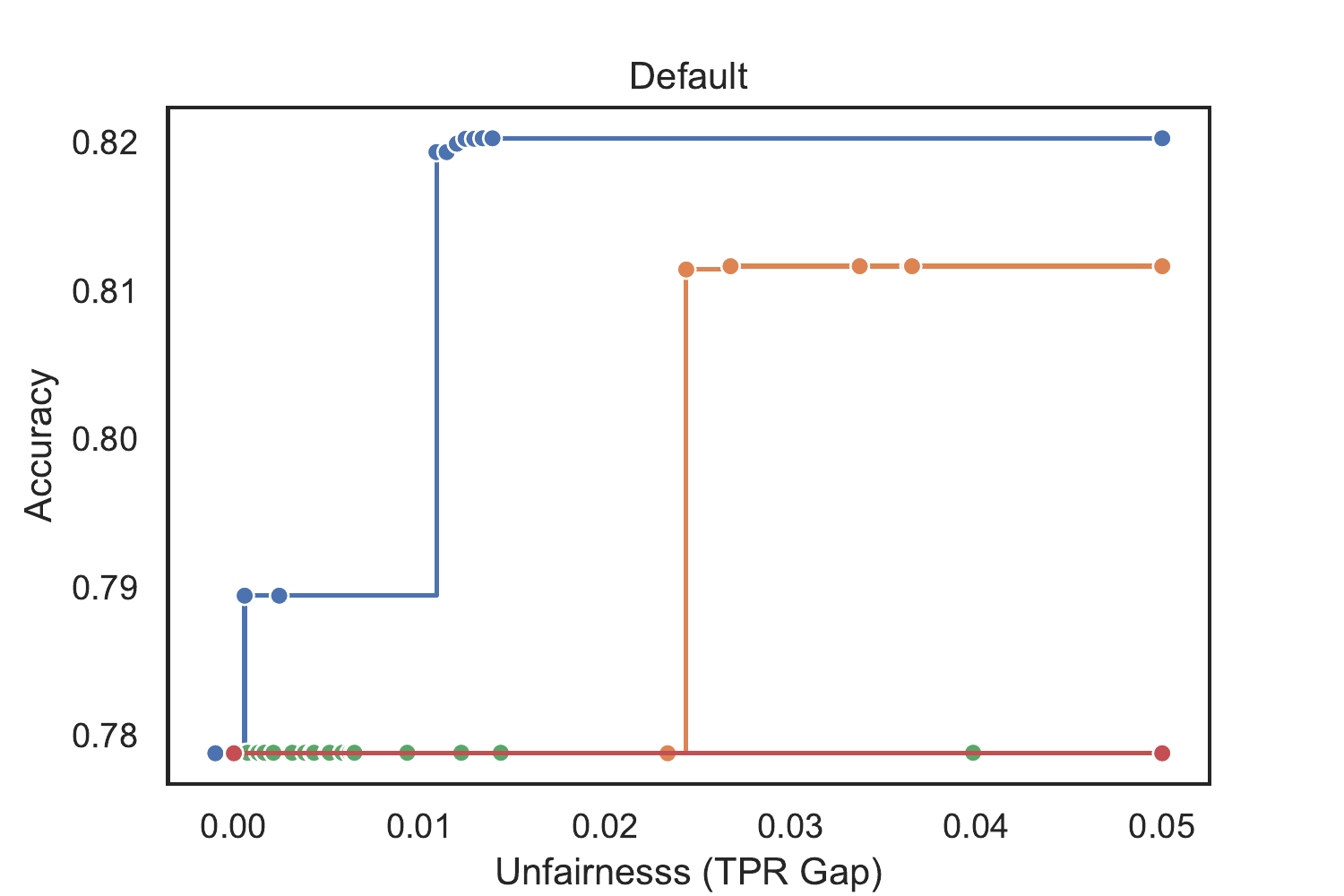}
        \end{subfigure}
        \begin{subfigure}
          \centering
          \includegraphics[width=0.4\linewidth]{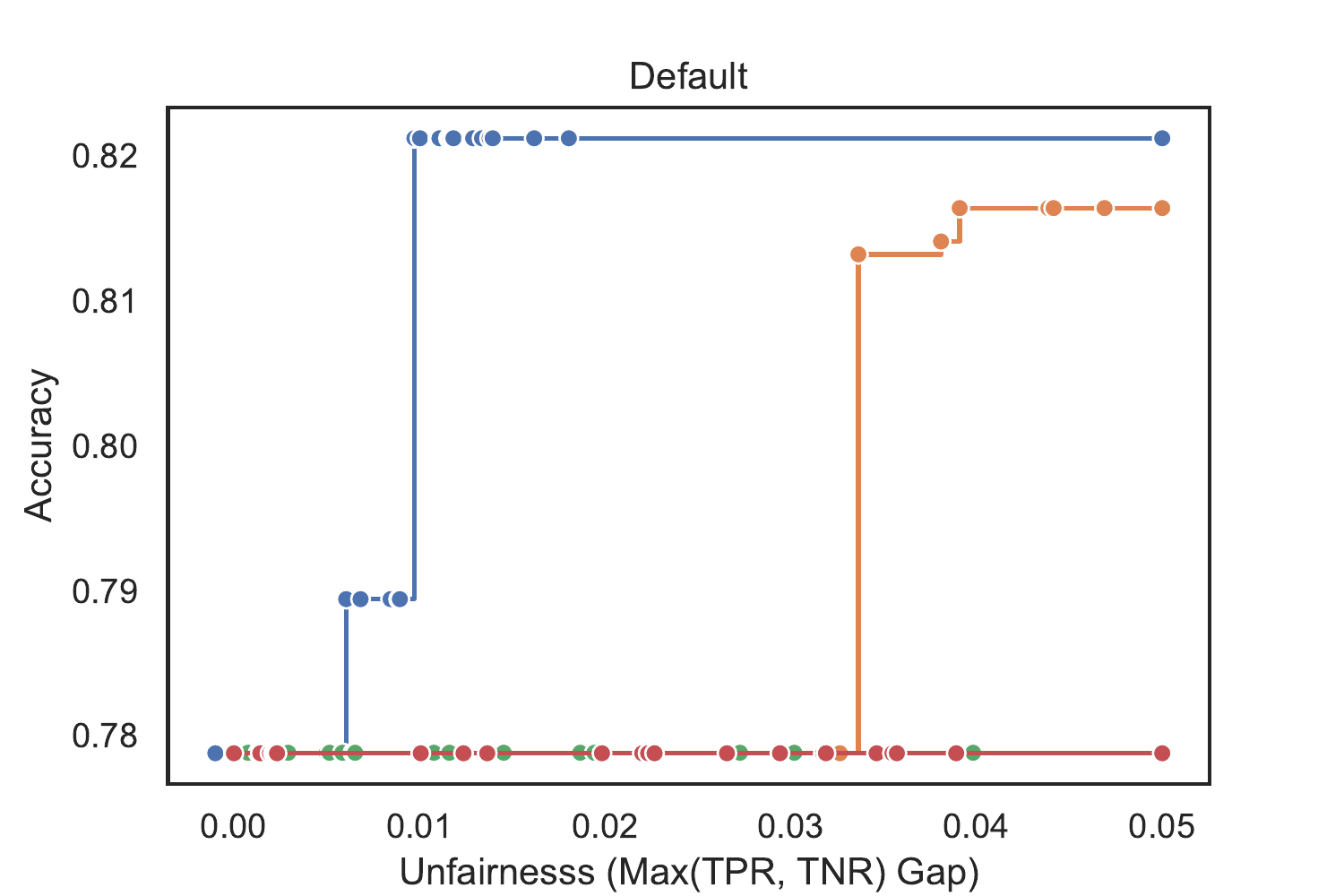}
        \end{subfigure}
        \begin{subfigure}
          \centering
          \includegraphics[width=0.4\linewidth]{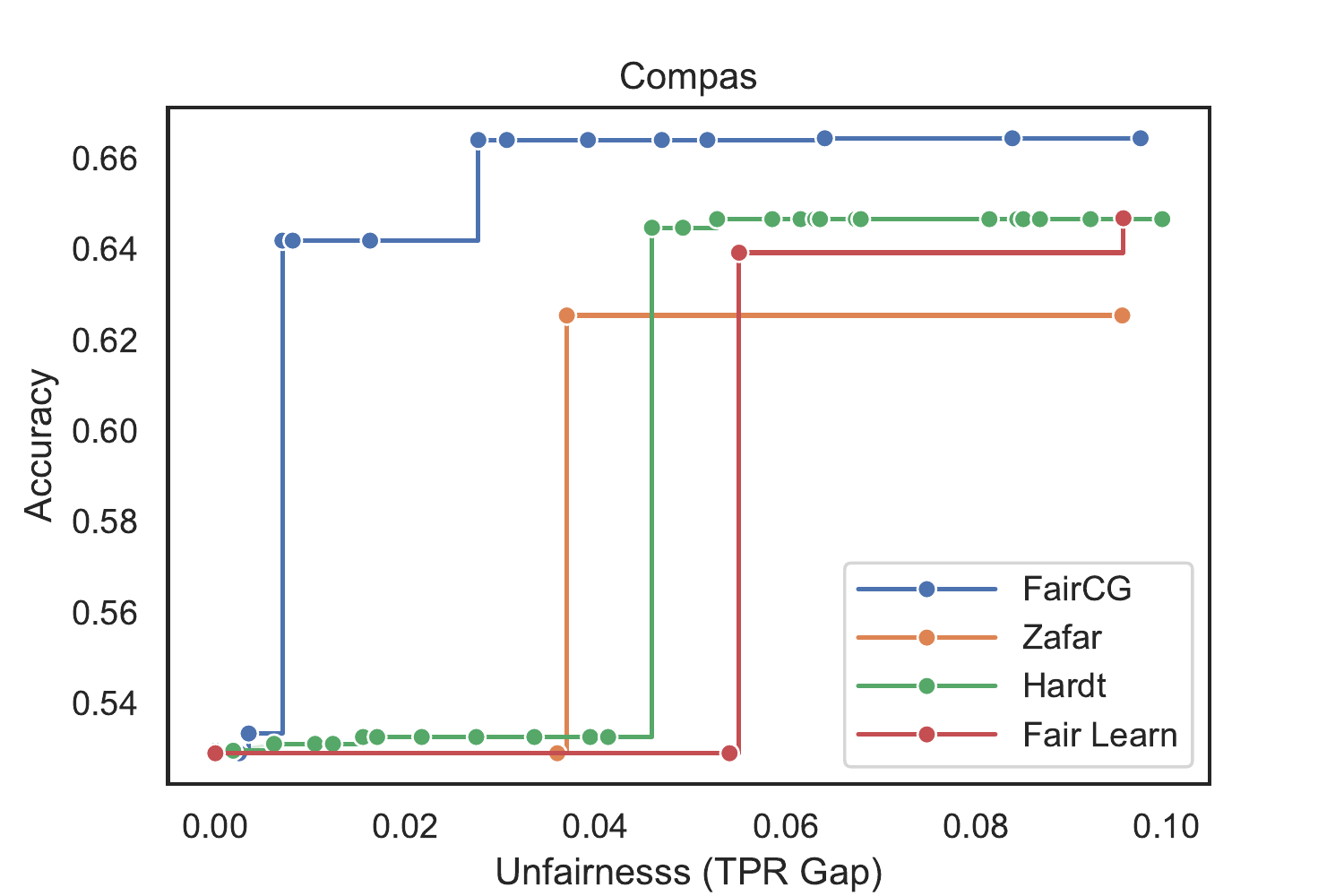}
        \end{subfigure}
    \begin{subfigure}
          \centering
          \includegraphics[width=0.4\linewidth]{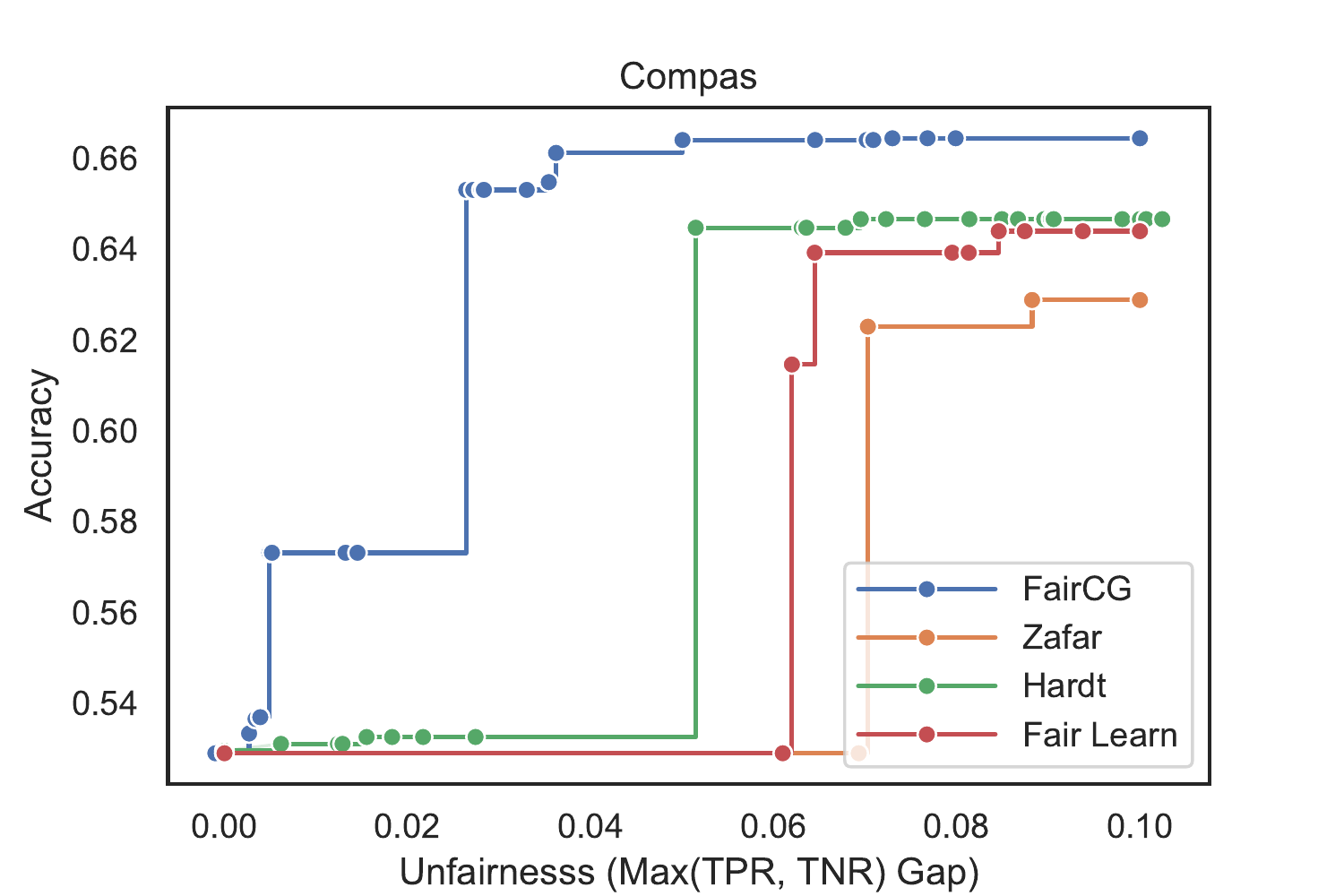}
        \end{subfigure}
         \caption {\label{benchmark_combined} Accuracy Fairness Frontier for Fair CG and other interpretable fair classifiers with respect to \emph{equality of opportunity} (left column) and \emph{equalized odds} (right column).}
    \end{figure*}

    \begin{figure}[t]
    \centering\footnotesize
        \begin{subfigure}
          \centering
          \includegraphics[width=0.6\textwidth]{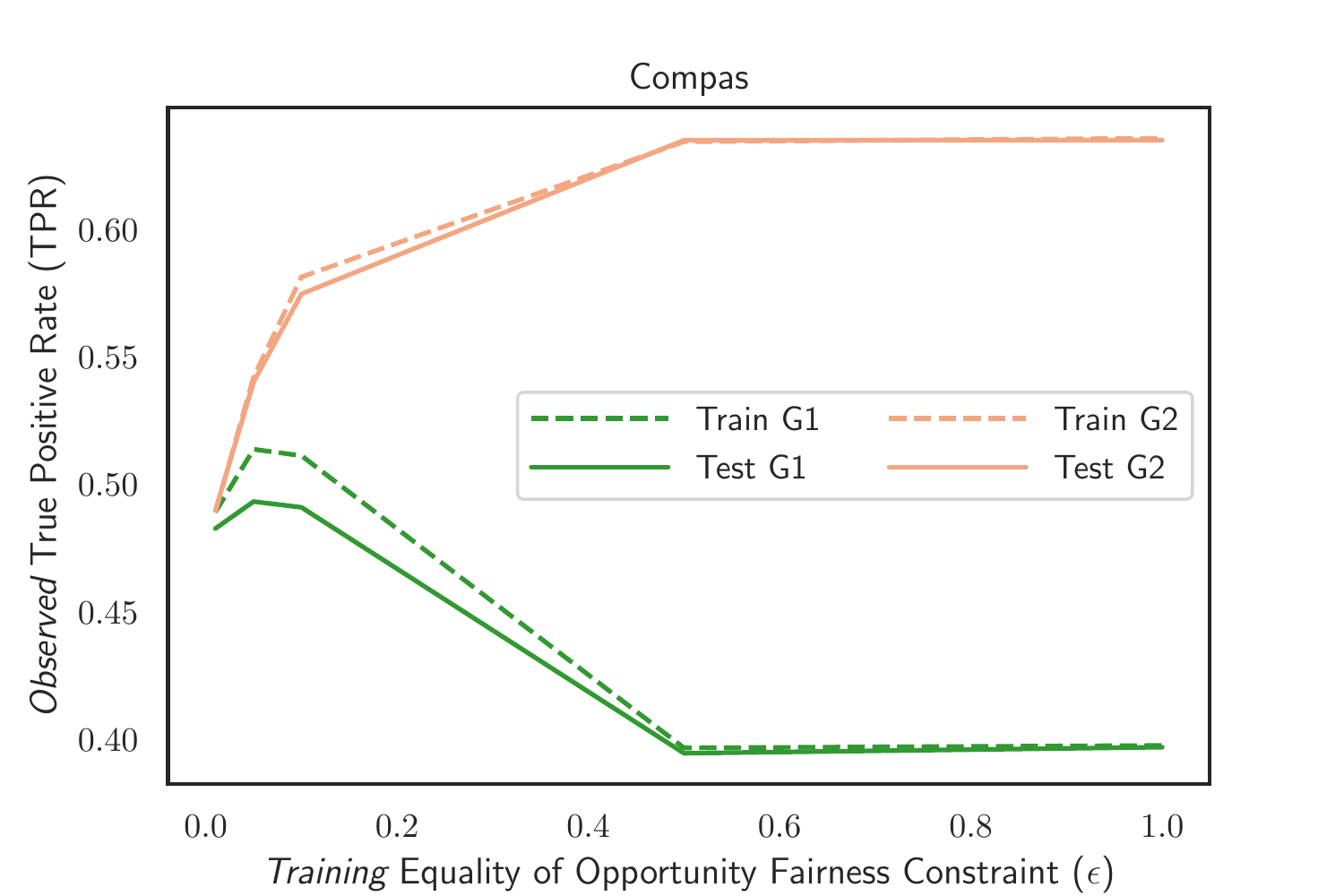}
        \end{subfigure}
         \caption {\label{diverging_tpr} Effect of relaxing the equality of opportunity fairness constraint on both train and test set true positive rate for each group.}\vskip-5mm
    \end{figure}

\begin{comment}
\begin{table*}[h]
\centering\footnotesize
\caption{\label{eqOptable} Mean accuracy and fairness Results for equality of opportunity}
\setlength{\tabcolsep}{5pt} % Default value: 6pt
\begin{tabular}{l   c c c c c c }		\toprule
& \multicolumn{2}{c}{Adult} & \multicolumn{2}{c}{Compas} & \multicolumn{2}{c}{Default} \\
&Accuracy&Fairness&Accuracy&Fairness&Accuracy&Fairness\\\midrule
Fair CG & 82.3 (0.8) & \textbf{0.7 (0.3)} & 53.3 (2.2) & \textbf{0.4 (0.1)} & 78.9 (1.5) & \textbf{0.5 (0.1)}\\\midrule
Zafar 2017 & 85.0 (0.6) & 2.7 (2.0) & 62.5 (1.5) & 3.7 (2.5) & 81.1 (0.7) & 2.4 (1.5) \\\midrule
Hardt 2016  & 76.7 (0.6) &1.7 (2.0) & 53.2 (2.5) & 1.6 (2.4) & - & -\\\midrule
Fair Learn& 82.1 (0.5) &4.3 (3.2) & 63.9 (1.8) & 5.5 (4.2) & - & -\\
\bottomrule
\end{tabular}%
\end{table*}%

   \begin{table*}[h]
\centering\footnotesize
\caption{\label{eqOdtable} Mean accuracy and fairness results for equalized odds}
\setlength{\tabcolsep}{5pt} % Default value: 6pt
\begin{tabular}{l c c c c c c }		\toprule
& \multicolumn{2}{c}{Adult} & \multicolumn{2}{c}{Compas} & \multicolumn{2}{c}{Default} \\
&Accuracy&Fairness&Accuracy&Fairness&Accuracy&Fairness\\\midrule
Fair CG & 77.8 (0.8) & 0.5 (0.2) & 53.6 (1.5) & \textbf{0.5 (1.0)} & 78.5 (0.8) & \textbf{0.8 (0.1)}\\\midrule
Zafar 2017& 85.0 (0.6) & 6.9 (3.1) & 62.3 (1.5) &8.6 (3.5) & 81.3 (0.5) & 3.6 (2.1) \\\midrule
Hardt 2016 & 76.0 (0.6) & \textbf{0.2 (0.01})  & 53.1 (0.01) & 0.6 (1.1) & - & -\\\midrule
Fair Learn & 80.4 (0.5) & 4.7 (4.5) & 61.5 (2.6) & 7.7 (4.2) & - & -\\
\bottomrule
\end{tabular}%
\end{table*}%
\end{comment}

\begin{table}[htb]
\centering
\caption{\label{complexity} Average complexity of rule sets}
\setlength{\tabcolsep}{5pt} % Default value: 6pt
\begin{tabular}{c c c c}		\toprule
$\epsilon$ & Adult & Compas & Default \\ \midrule 
0.01 & 59.9 & 11.2 & 11.4 \\ \midrule 
0.1 & 60.1 & 10.9 & 9.3 \\ \midrule 
0.5 & 50.5 & 9.2 & 9.3 \\ \midrule
\bottomrule
\end{tabular}%
\end{table}%
\subsection{Interpretability}
In our framework, the complexity of the rule set (defined as the number of rules plus the number of conditions in each rule) is used as a measure for the interpretability of the rule set. Table \ref{complexity} summarizes the mean complexity of the rule set selected for each dataset and some sample $\epsilon$ under the equality of opportunity constraint. 

To give a sense of the interpretability of FairCG classifiers, we generated sample rule sets to predict criminal recidivism on the compas dataset with and without fairness constraints. We trained each rule set on one train/test split of the data set, and report the rule set as well as its out of sample accuracy and fairness. A sample rule set without any fairness constraints is:

\medskip\noindent\text{Predict repeat offence if:}
\begin{center}
\big[\text{(Score Factor$=$True) and (Misdemeanor=False)  }\big] \\ \text{~OR~} 
\\\big[\text{(Race$\neq$Black) and (Score Factor$=$True) and (Misdemeanor$=$True)} \\ \text{ and (Age$<$45) and (Gender$=$'Male')}\big] \\
 \text{~OR~}
\\\big[\text{(Race$=$Black) and (Score Factor$=$True) and (Priors$\geq$ 10)} \\ \text{ and (Age$<$45) and (Gender$=$'Male')}\big] 
\end{center}

This rule set has a test set accuracy of $67.2\%$, but a $20\%$ gap in the false negative rate, and $22\%$ gap in the false positive rate between the two groups respectively (i.e. $22\%$ unfair with respect to equalized odds). Adding an equalized odds constraint of $0.05$, is:

\medskip\noindent\text{Predict repeat offence if:}
\begin{center}
\big[\text{(Race$\neq$Black) and (Score Factor$=$True) and (Age $<$ 45)}\big] \\ \text{~OR~}
\\\big[\text{(Race$=$Black) and (Score Factor$=$True) and (Misdemeanor$=$True)} \\ \text{ and (Age$<$45) and (Gender$=$'Male')}\big] 
\end{center}

The fair rule set is less accurate, it has a test set accuracy of $65.1\%$, but a $1\%$ gap in the false negative rate and $3.4\%$ gap in the false positive rate between the two groups respectively. Both rule sets are arguably quite interpretable - they have a small number of clauses, each with a small number of conditions. The fair rule set is also interesting in that it highlights the the bias of compas's recidivism tool (score factor indicates whether or not the compas tool predict the offender a high risk of recidivism). In our optimal rule set, the compas tool seems to be a good predictor of recidivism for white offenders, but for Black offenders our rule set adds a number of additional criteria including requiring the offender to be young, male, and have a high number of priors. One lackluster aspect of our rule set, is the fact that it never predicts a repeat offence if the offender if a Black woman. This is a consequence of the relatively low rate of Black woman re-offenders in the data set (3 percent of the overall data). In fact, the error rate for Black women (30 percent) is below the average error rate in the test data (37 percent). However, the false negative error rate is 100 percent and thus could present fairness violations under equalized odds. To account for this, our formulation can easily be extended to include both race and gender as sensitive attributes (i.e. one group for Black men, Black women, white men, and white women respectively) by adding additional fairness constraints to bound the discrepancy between each set of groups. 

\noindent

\section{Conclusion}
In this paper we introduced an IP formulation for building fair Boolean rule sets under both equality of opportunity and equalized odds. Experimental results on classic fair machine learning datasets validated that our algorithm is competitive with the state of the art - dominating popular fair classifiers on 2 of 3 datasets, and remaining unbeaten in regimes of strict fairness. Overall, our algorithm Fair CG provides a powerful tool for practitioners that need simple, fair, and interpretable models for machine learning in socially sensitive settings. 

    \newpage
    \bibliographystyle{plain}
    \bibliography{references.bib}
    
    \newpage
    \appendix

    \section{Datasets and data processing}
In our experiments, we use three common fair machine learning datasets: adult, compas and default. Both adult\footnote{https://archive.ics.uci.edu/ml/datasets/adult} and default\footnote{https://archive.ics.uci.edu/ml/datasets/default} can be found on the UCI machine learning dataset repository \cite{UCI}. Both datasets were unchanged from the data available at the referenced links. For the adult dataset we just use the training data provided, and for default we use the the entire dataset provided. For the compas data from ProPublica \cite{propub} we use the fair machine learning cleaned dataset\footnote{https://www.kaggle.com/danofer/compass}. Following the methodology of \cite{zafar2015fairness} we also restrict the data to only look at African American and Caucasian respondents - filtering all datapoints that belong to other races and creating a new binary column  which indicates whether or not the respondent was African American. A summary of our data and the sensitive attributes we use for group identification is included below.

\begin{table}[h]
\centering
\caption{\label{datasets} Overview of datasets}
\setlength{\tabcolsep}{5pt} % Default value: 6pt
\begin{tabular}{c c c c}		\toprule
Dataset & Examples & Features & Sensitive Variable \\ \midrule 
Adult & 32,561 & 14 & Gender \\
Compas & 5,278 & 7 & Race \\
Default & 30,000 & 23 & Gender (X2 column) \\ 
\bottomrule
\end{tabular}%
\end{table}%

   To convert the data to be binary-valued, we use a standard methodology also used in \cite{dash2018boolean}. 
    For categorical variables $j$ we use one-hot encoding to binarize each variable into multiple indicator variables that check $X_{j} = x$, and the negation $X_{j} \neq x$.
    For numerical variables we compare values against a sequence of thresholds for that column and include both the comparison and it's negation (i.e. $X_j \leq 1, X_j \leq 2$ and $X_j > 1, X_j > 2$). For our experiments we use the sample deciles as the thresholds for each column. We use the binarized data for all the algorithms we test to control for the binarization method.

\section{Computing infrastructure}
Our experiments were run on a windows desktop computer with an Intel core i7-9000 3 GHz processor and 32 GB of RAM. Our python environment was configured with anaconda and included the following package dependencies:

\begin{table}[h]
\centering
\caption{\label{packages} Overview of package dependencies}
\setlength{\tabcolsep}{5pt} % Default value: 6pt
\begin{tabular}{c c}		\toprule
Package & Version \\ \midrule 
Python & 3.7.6 \\
Gurobi & 9.0.1 \\
Numpy & 1.18.1 \\
Pandas & 1.0.0 \\
\bottomrule
\end{tabular}%
\end{table}%

\section{Hamming loss}
As part our experiments we tested the impact of optimizing for Hamming Loss instead of 0-1 loss directly. To adapt our IP formulation to optimize for accuracy we have to both change the objective function to include all errors, and add additional constraints to track the number of misclassified data points in the negative class. The full IP formulation is:

\begin{align}
\textbf{min}&\quad&\sum_{i\in \cal{P}} \zeta_i +\sum_{i\in \cal{N}} \zeta_i\\
	\textbf{s.t.}  
	&&\zeta_i + \sum_{k\in {\cal K}_i} w_k&\geq 1 , \quad \zeta_i \geq 0,~i \in \cal{P}~~ \\
	&& \sum_{ k \in {\cal K}_i} w_k  &\leq   \frac{C}{2} \zeta_i , ~ i \in {\cal N}, ~~ \label{const:accZ} \\
	&&\sum_{k\in \mathcal{K}}  c_k w_k &\leq C ~~ \label{const:compZ}\\%[.3cm]
	&&\frac{1}{\abs{\mathcal{P}_1}}\sum_{i \in  \mathcal{P}_1 } \zeta_i 
	        - &\frac{1}{\abs{\mathcal{P}_2}}\sum_{i \in  \mathcal{P}_2} \zeta_i \leq \epsilon_1 \label{const:heo1}\\
	&&\frac{1}{\abs{\mathcal{P}_2}}\sum_{i \in  \mathcal{P}_2} \zeta_i 
	        - &\frac{1}{\abs{\mathcal{P}_1}}\sum_{i \in  \mathcal{P}_1} \zeta_i \leq \epsilon_1 \label{const:heo2} \\
	&&w_k  &\in\{0,1\}, \quad {k\in \cal{K}}
\end{align}

where constraint (\ref{const:accZ}) is the new constraint to track misclassifications for the negative class. We upper bound $\sum_{ k \in {\cal K}_i} w_k$ with $C/2$, as the minimum complexity of a rule is 2 and thus the complexity constraint (\ref{const:compZ}) implies the upper bound on the number of rules is $C/2$.  limit on the number of rules is $C/2$. Instead of using disaggregated  constraints (i.e. $w_k \leq \zeta_i$), we use the aggregated form as modern solvers will handle it more efficiently \cite{IPref}. Constraints (\ref{const:heo1}) and (\ref{const:heo2}) bound the equality of opportunity gap. Note that we focus on the equality of opportunity fairness metric, as the equalized odds constraints are specifically tailored for the Hamming Loss model. 

In our experiments comparing the Hamming loss and 0-1 loss formulations, we used a fixed set of  rules generated from our column generation procedure (on an average of 6781, 1007, and 7584 rules per fold for the adult, compas, and default datasets respectively). The fixed rule sets were mined during our experiments on FairCG (i.e. the first stage of the procedure outlined in the computational approach section of the paper). Using these fixed rule sets, we ran both master IP models with the same fairness parameter ($\epsilon$) and the following ranges of complexity bounds:

\begin{itemize}
    \item \textbf{Adult}: $\{10, 20, 30, 40\}$
    \item \textbf{Compas}: $\{10, 15, 20, 30 \}$
    \item \textbf{Default}: $\{10, 15, 20, 30\}$
\end{itemize}

A time limit of 20 minutes was set to solve each IP model, though no instance reached the time limit and all IPs were solved to optimality. We used 10-fold cross validation to select the best complexity bound for each model on each dataset. Figure \ref{hamming} summarizes the relative performance of each model on each dataset with different $\epsilon$ values. We can see that on all datasets the Hamming loss model is able to achieve practically indistinguishable performance in terms of both train and test set accuracy at a fraction of the computation time. Specifically, while the accuracy model appears to perform slightly better with respect to train accuracy, the trend is reversed when it comes to test accuracy. This suggests Hamming loss may lead to rule sets with better generalization. This is especially pronounced on larger datasets, such as Adult, where the accuracy model takes nearly 100x the computation time. Table \ref{Hamming_table} summarizes the same results numerically. Jupyter notebook with the implementation of this experiment can be found in our code base.

\section{Rule mining} \label{sec:ruleMine}
We compare the rules generated from our column generation procedure to rules mined from other tree-based models. For each tree-based model we train the classifier using the split binarized data and use a range of hyperparameters (outlined in Table \ref{ruleMine_hp}) and extract every decision rule present. To extract the rules we simply look at the decision path used for each leaf node which predicts a positive response and translate that to which binarized features would be included in the rule. An implementation of the function used to do the rule extraction is included in our code base. 

\begin{table}[h]
\centering
\caption{\label{ruleMine_hp} Overview of Hyperparameters Used for Rule Mining}
\setlength{\tabcolsep}{5pt} % Default value: 6pt
\begin{tabular}{c c c}		\toprule
Method & Hyperparameter & Values \\ \midrule 
Decision Tree & Max Depth & \{1,3,5,..30\} \\ \midrule
\multirow{2}[2]{*}{Random Forest} & Max Depth &  \{1,3,5,..30\} \\ 
 & Number of Trees &  \{1,2,3,..10\} \\ \midrule
\multirow{2}[2]{*}{Fair Decision Tree}  & Max Depth & \{1,3,5,..30\}  \\ 
 & Fairness $\epsilon$ & \{0.01,0.03,...0.5\}  \\ 
\bottomrule
\end{tabular}%
\end{table}%

We leverage three different tree-based models: Decision Trees, Random Forests, and Fair Decision Trees from fairlearn \cite{agarwal2018reductions}. For the decision tree and random forest models we use the scikit-learn implementations \cite{scikit-learn}. For decision trees we vary the depth of tree allowed, and for random forests we adjust both the tree depth and the number of trees included in the forest. For the fair learn models we use the exponentiated gradient approach and use decision trees as the base learner. We vary both the max depth of the trees used and how strict the fairness criteria is (i.e. $\epsilon$). The average size of the rule set generated for each dataset and fold is outlined in Table \ref{ruleset_sze}.

\begin{table}[h]
\centering
\caption{\label{ruleset_sze} Average Rule set Size}
\setlength{\tabcolsep}{5pt} % Default value: 6pt
\begin{tabular}{c c c c}		\toprule
Method & Adult & Compas & Default \\ \midrule 
Decision Tree & 2276 & 317 & 2558 \\ \midrule 
Random Forest & 15083 & 3309 & 24680 \\ \midrule 
Fair DT & 4321 & 1034 & 4958 \\ \midrule
FairCG & 6781 & 1071 & 7584 \\ 
\bottomrule
\end{tabular}%
\end{table}%

To test the efficacy of each rule set we run the master IP model with each rule set, and do 10-fold cross validation to select the best rule set complexity for each model. We repeat the tests for three different fixed fairness criteria ($\epsilon$). %Figures \ref{rule_ming_eqOp} and \ref{rule_ming_eqOd} show the results on our three datasets under different $\epsilon$ for the equality of opportunity and equalized odds constraints respectively. 
Tables \ref{tab_ruleset_adult}, \ref{tab_ruleset_compas} and \ref{tab_ruleset_default} summarize the results numerically. It is important to emphasize that the performance reported for each rule set is not the performance of that model on the datasets (ex. the accuracy of running a decision tree model), but rather the performance of the mined rule set from that model when optimized within our IP framework (ex. the accuracy of the optimal rule set using rules mined from a decision tree). Across both metrics and all three datasets we see that the rule set generated by column generation performs better than those from other heuristic tree sources - outperforming them with respect to Hamming loss, train set accuracy and test set accuracy.

\section{Fairness-accuracy trade-offs}

For our experiments we took a two-phase approach. During the first rule generation phase, we ran our column generation algorithm with a set of different hyperparameters to generate a set of potential rules. To warm start this procedure, we also start the column generation process with a set of rules mined from a random forest classifier (as discussed in section \ref{sec:ruleMine}). We then solve the master IP with the set of candidate rules and a larger set of hyperparameters to generate the curves included in the body of the report. Tables \ref{eps_hp} and \ref{c_hp} summarize the hyperparameters used for both Phase I and II. Note that for the equalized odds formulation, we set $\epsilon_1 = \epsilon_2$ and use the values in Table \ref{eps_hp}.

\begin{table}[h]
\centering
\caption{\label{eps_hp} Overview of $\epsilon$ Hyperparameters Tested}
\setlength{\tabcolsep}{5pt} % Default value: 6pt
\begin{tabular}{c c c c c}		\toprule
Dataset & Phase 1 & Phase 2  \\ \midrule 
Adult &\{0.01, 0.1, 1\}  & \{0, 0.01, 0.05, 0.1, 0.15, 0.2, 1\} \\
Compas &\{0.01, 0.1, 1\}  & \{0, 0.01, 0.05, 0.1, 0.15, 0.2, 1\} \\
Default &\{0.01, 0.1, 1\}  & \{0, 0.01, 0.03, 0.05, 0.1, 0.2, 1\} \\
\bottomrule
\end{tabular}%
\end{table}%

\begin{table}[h]
\centering
\caption{\label{c_hp} Overview of $C$ Hyperparameters Tested}
\setlength{\tabcolsep}{5pt} % Default value: 6pt
\begin{tabular}{c c c c c}		\toprule
Dataset & Phase 1 & Phase 2  \\ \midrule 
Adult &\{5, 20, 40\}  & \{5, 15, 20, 30, 40\} \\
Compas &\{5, 15, 30\}  & \{5, 10, 15, 20, 30\} \\
Default &\{5, 15, 30\}  & \{5, 10, 15, 20, 30\} \\
\bottomrule
\end{tabular}%
\end{table}%

We tested our algorithm against three other popular interpretable fair classifiers: Zafar 2017 \cite{Zafar_2017}, Hardt 2016 \cite{hardt2016equality}, and Fair Decision trees trained using the exponentiated gradient algorithm in fairlearn \cite{agarwal2018reductions}. For the Zafar algorithm we used the optimizer parameters specified in Table \ref{zafar_hp}, and tested a range of 30 different $\epsilon$ values (linearly spaced between 0 and 0.5) for the covariance threshold. For Hardt 2016 we used the logistic regression implementation from scikit-learn \cite{scikit-learn} and tested 100 different decision thresholds for each sub-group (1\% increments). Finally, for the exponentated gradient algorthm from fairlearn we used scikit-learn's decision tree as the base estimator and tested both a range of maximum depth hyperparameters (20 values linearly spaced between 1 and 30), and 30 different $\epsilon$ values (linearly spaced between 0 and 0.5) for the fairness constraints. For all the algorithms we used 10-fold cross-validation to select the best hyperparameters for every level of fairness and removed dominated points from the figures present in the results section of the main paper.

\begin{table}[h]
\centering
\caption{\label{zafar_hp} Overview of Zafar Optimizer Hyperparameters}
\setlength{\tabcolsep}{5pt} % Default value: 6pt
\begin{tabular}{c c c}		\toprule
Dataset & $\tau$ & $\mu$  \\ \midrule 
Adult & 5 & 1.2\\
Compas &20  & 1.2\\
Default & 0.5 & 1.2 \\
\bottomrule
\end{tabular}%
\end{table}%

\begin{figure}[!htb]
    \centering
    \begin{subfigure}
      \centering
      \includegraphics[width=0.95\textwidth]{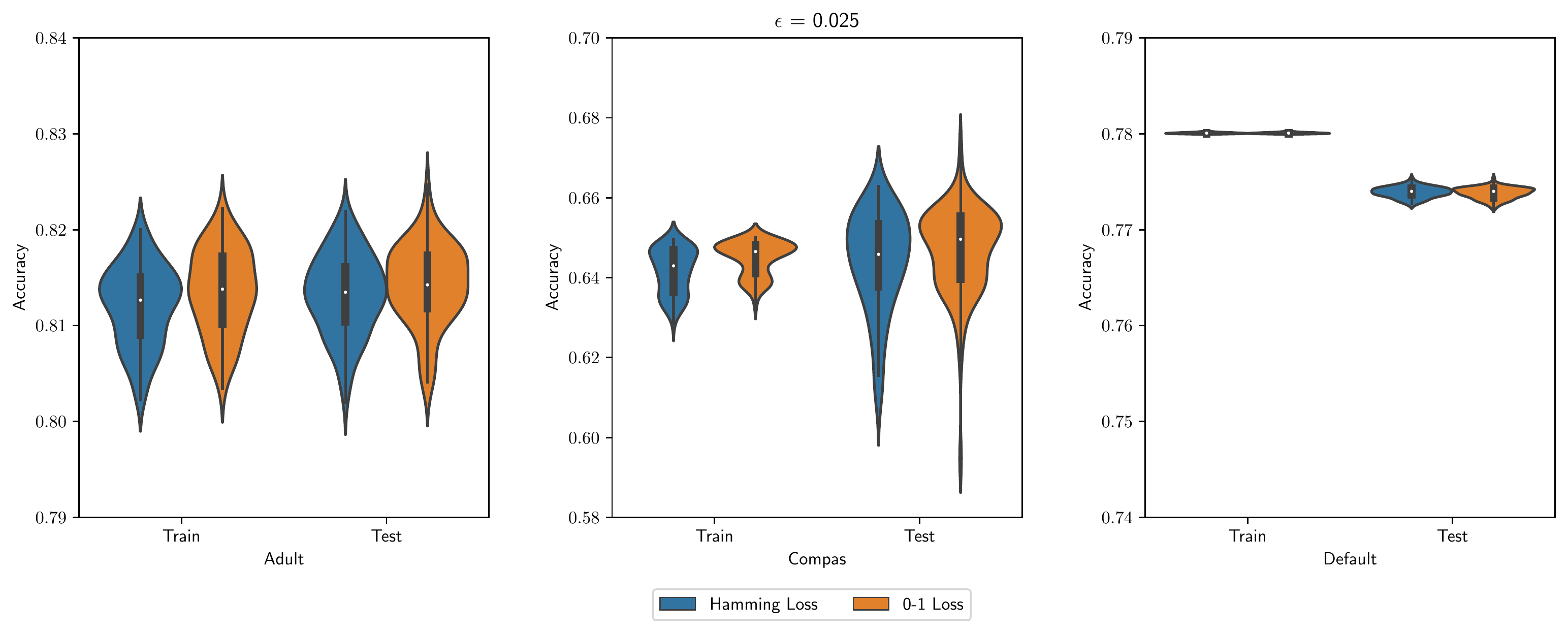}
    \end{subfigure}
        \begin{subfigure}
      \centering
      \includegraphics[width=0.95\textwidth]{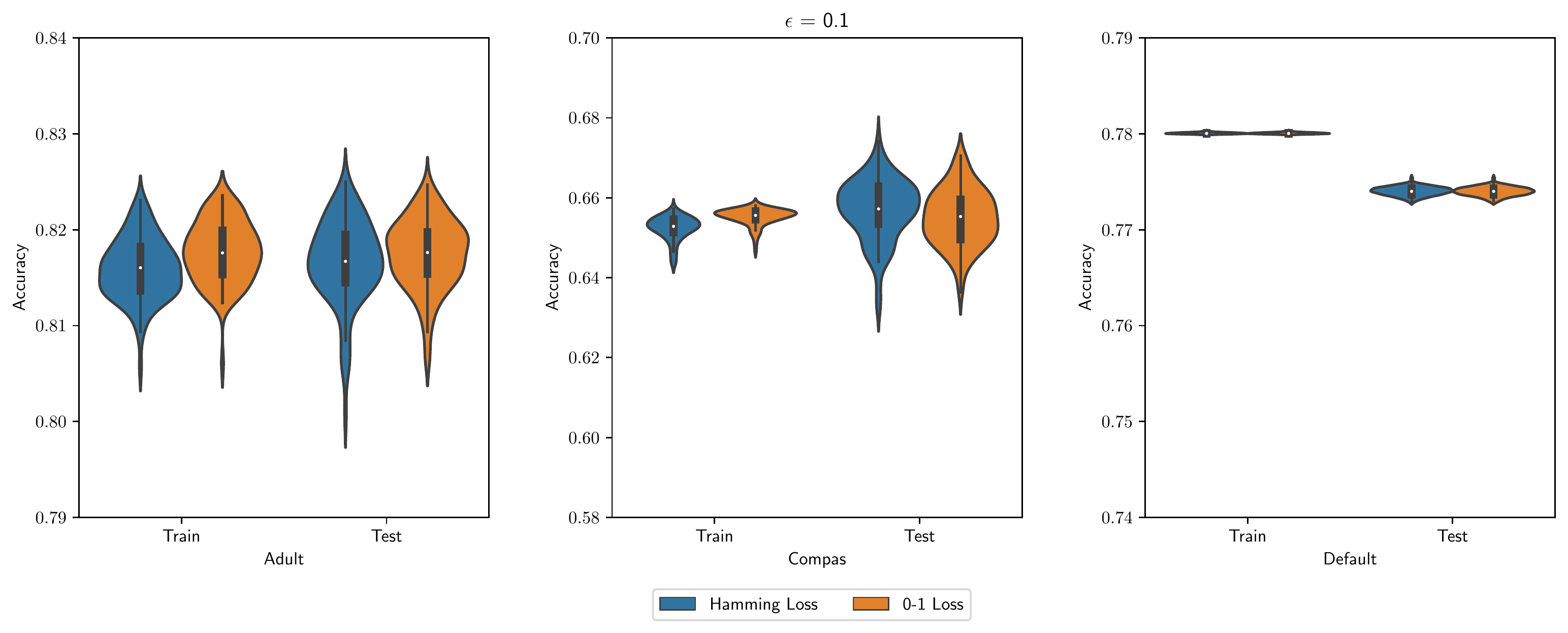}
    \end{subfigure}
    \begin{subfigure}
      \centering
      \includegraphics[width=0.95\textwidth]{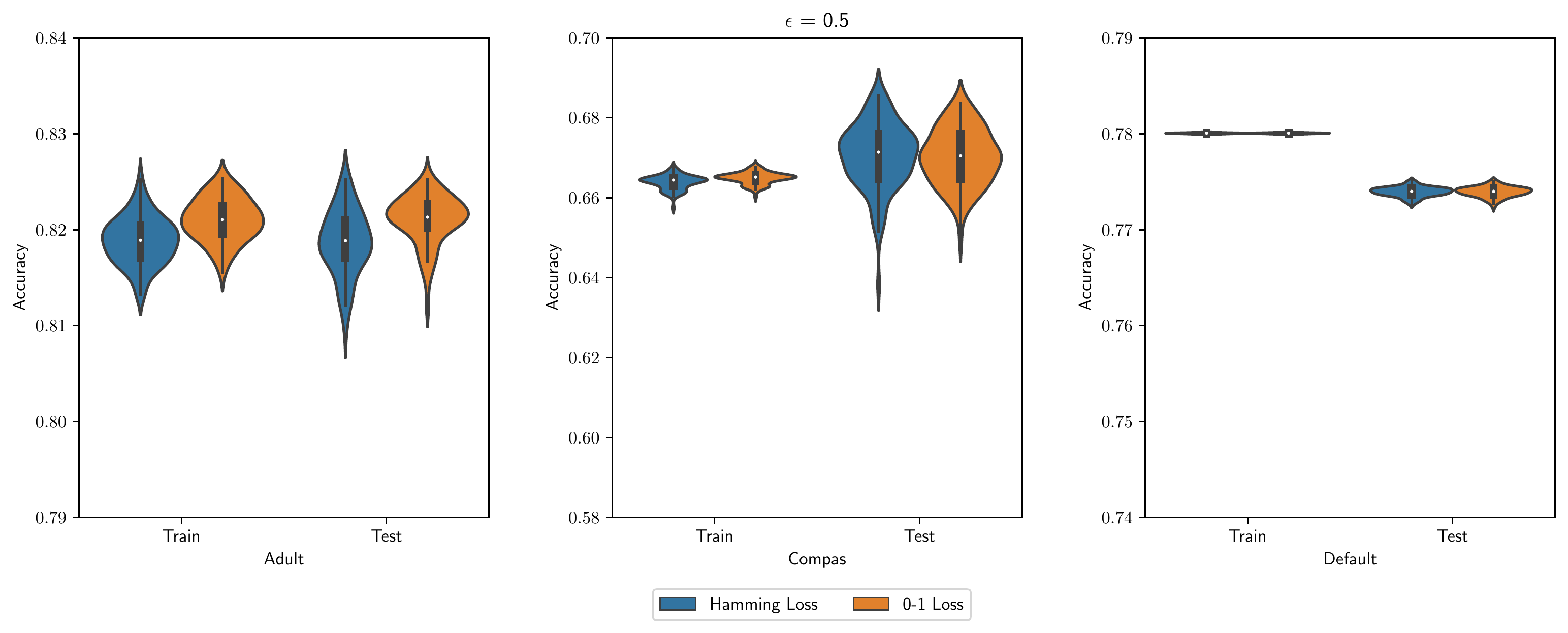}
    \end{subfigure}

  \caption {\label{hamming} Relative performance of models trained to maximize accuracy or minimize Hamming loss respectively under equality of opportunity constraint ($\epsilon = 0.025$). The violin plots show the distribution of train and test set accuracy for models under each objective for 100 random splits of the datasets.}
\end{figure}

\renewcommand{\baselinestretch}{1.2}

\begin{table*}[h]
\centering
\caption{\label{Hamming_table} Summary of performance difference when optimizing for Hamming loss vs. accuracy (standard deviation in parenthesis)}
\small
\setlength{\tabcolsep}{5pt} % Default value: 6pt
\begin{tabular}{c c c c c c} \toprule
Dataset & $\epsilon$ & Objective & Train Acc. & Test Acc. & Comp. Time (s) \\ \midrule
\multirow{6}{*}{Adult} & \multirow{2}[2]{*}{0.025} & Accuracy & 81.24 (1.1)& 80.8 (1.3) & 573 (73) \\
& & Hamming & 81.21 (0.2) & 81.4 (0.5) & 31 (5) \\
& \multirow{2}{*}{0.1} & Accuracy & 81.8 (0.8)& 81.2 (1.1) & 565 (63)  \\
& & Hamming & 81.7 (0.2) & 81.6 (0.7) & 30 (6) \\
& \multirow{2}{*}{0.5} & Accuracy &81.8 (0.7) & 81.3 (0.1) & 568 (70) \\
& & Hamming & 81.8 (0.2) & 81.7 (0.7) & 30 (4)  \\ \midrule
\multirow{6}{*}{Compas} & \multirow{2}[2]{*}{0.025} & Accuracy & 65.4 (0.3) & 63.6 (3.4) & 24 (3)  \\ 
& & Hamming & 65.2 (0.3) & 64.7 (2.2)& 6 (3)  \\
& \multirow{2}{*}{0.1} & Accuracy & 66.3 (0.2) & 65.1 (2.7) & 21 (4)  \\
& & Hamming & 66.1 (0.2) & 65.5 (2.6) & 4 (2)  \\
& \multirow{2}{*}{0.5} & Accuracy & 67.1 (0.3) & 65.9 (2.9) & 6 (1) \\
& & Hamming & 66.9 (2.5) & 66.0 (2.3) & 2 (0.3)  \\ \midrule 
\multirow{6}{*}{Default} & \multirow{2}[2]{*}{0.025} & Accuracy &77.9 (0) & 77.9 (0)& 511 (105) \\
& & Hamming &77.9 (0) &77.9 (0) & 3 (0.2)  \\
& \multirow{2}{*}{0.1} & Accuracy &77.9 (0)  &77.9 (0)  & 475 (117)  \\
& & Hamming &77.9 (0)  &77.9 (0)  & 4 (0.6)  \\
& \multirow{2}{*}{0.5} & Accuracy & 77.9 (0) &77.9 (0)   & 519 (119)\\
& & Hamming &77.9 (0)  &77.9 (0)  & 3 (0.3)  \\
\bottomrule
\end{tabular}%
\end{table*}%

\begin{table*}[h]
\centering
\caption{\label{tab_ruleset_adult} Summary of rule set performance on adult dataset (standard deviation in parenthesis)}
%\rotatebox{90}
{\small
\setlength{\tabcolsep}{5pt} % Default value: 6pt
\begin{tabular}{c c l l l l l l } \toprule
 &  & \multicolumn{3}{c}{Equality of Opportunity} & \multicolumn{3}{c}{Equalized Odds} \\
 $\epsilon$ & Rule Set & Train Acc. & Test Acc. & Hamming loss & Train Acc. & Test Acc. & Hamming loss \\ \midrule
\multirow{6}[2]{*}{0.025} & Decision Tree & 77.1 (1.5)& 77.0 (1.5) & 6703 (434) & 76.9 (1.4) & 76.7 (1.4) & 6781 (418)\\
& Fair Decision Tree & 76.5 (0.6) & 76.4 (0.4) & 6883 (188) &76.5 (1.3) & 76.5 (1.4) & 6900 (369)\\
& Random Forest & 76.5 (0.8) & 76.5 (1.7) & 6885 (243) & 76.1 (0.2) & 76.1 (1.0) & 7000 (64)\\
& All but CG & 77.1 (1.5)& 77.0 (1.5) & 6703 (434) & 77.0 (1.3) & 76.8 (1.4) & 6781 (418) \\
& Fair CG & 80.7 (1.2) & 80.5 (1.1) & 5693 (358) &81.0 (1.6) & 81.0 (1.4) & 5567 (438) \\ 
& All & 80.7 (1.2) & 80.5 (1.1) & 5693 (358) & 81.1 (1.4) & 81.0 (1.4) & 5566 (438) \\ \cmidrule{1-8}
\multirow{6}[2]{*}{0.1} & Decision Tree & 77.6 (1.4)& 77.3 (1.3) & 6558 (402) & 77.4 (1.4) & 77.2 (1.3) & 6610 (414)  \\
& Fair Decision Tree & 77.9 (1.4) & 77.8 (1.6) & 6458 (412) & 78.7 (0.2) & 78.7 (0.8) & 6244 (68)\\
& Random Forest & 76.8 (0.8) & 76.5 (1.0) & 6780 (385) & 76.7 (0.8) & 76.5 (1.0) & 6839 (232) \\
& All but CG & 77.9 (1.4) & 77.8 (1.6) & 6458 (412) & 78.7 (0.2) & 78.7 (0.8) & 6244 (68) \\
& Fair CG & 81.8 (0.2) & 81.6 (0.7) & 5340 (67) & 81.3 (1.3) & 81.4 (1.1) & 5450 (291) \\
& All & 81.9 (0.4) & 81.3 (0.5) & 5338 (80) & 81.6 (1.0) & 81.3 (1.1) & 5417 (259)\\ \cmidrule{1-8}
\multirow{6}[2]{*}{0.5} & Decision Tree & 77.6 (1.4)& 77.3 (1.3) & 6561 (399) & 77.4 (1.4) & 77.2 (1.3) & 6614 (409) \\
& Fair Decision Tree & 79.8 (1.1) & 79.7 (1.5) & 5926 (326) & 78.7 (1.8) & 78.6 (1.1) & 6253 (518) \\
& Random Forest & 77.0 (1.0) & 76.9 (1.5) & 6743 (305) & 76.9 (1.3) & 76.7 (1.4) & 6769 (372)\\
& All but CG & 79.8 (1.1) & 79.7 (1.5) & 5926 (326) & 78.7 (0.2) & 78.7 (0.8) & 6244 (68) \\
& Fair CG & 81.8 (0.2) & 81.6 (0.7) & 5339 (67) & 81.4 (1.1) & 81.4 (1.3) & 5422 (260)\\
& All &  81.9 (0.3) & 81.3 (0.9) & 5335 (62) & 81.6 (1.0) & 81.3 (1.1) & 5417 (259)\\ 
\bottomrule
\end{tabular}%
}
\end{table*}%

\begin{table*}[h]
\centering
\caption{\label{tab_ruleset_compas} Summary of rule set performance on compas dataset (standard deviation in parenthesis)}
%\rotatebox{90}
{\small
\setlength{\tabcolsep}{5pt} % Default value: 6pt
\begin{tabular}{c c l l l l l l } \toprule
 &  & \multicolumn{3}{c}{Equality of Opportunity} & \multicolumn{3}{c}{Equalized Odds} \\
$\epsilon$ & Rule Set & Train Acc. & Test Acc. & Hamming loss & Train Acc. & Test Acc. & Hamming loss \\ \midrule
\multirow{6}[2]{*}{0.025} & Decision Tree & 57.5 (3.7)& 57.2 (4.4) & 2020 (175) & 60.3 (2.8) & 60.2 (4.1) & 1884 (131) \\
& Fair Decision Tree & 60.2 (0.5) & 59.4 (0.5) & 1891 (223) & 62.1 (1.7) & 62.5 (3.0) & 1802 (82)\\
& Random Forest & 62.8 (2.2) & 62.4 (3.2) & 1782 (116) & 64.0 (1.1) & 63.7 (2.9) & 1710 (51) \\
& All but CG & 62.8 (2.1) & 62.6 (2.9) & 1750 (115) & 64.0 (1.1) & 63.7 (2.9) & 1710 (51) \\
& Fair CG & 64.5 (0.7) & 64.4 (2.0) & 1690 (34) & 65.0 (0.3) & 64.6 (2.3) & 1660 (14)\\
& All & 64.6 (0.9) & 64.3 (2.2) & 1679 (44) & 65.3 (0.2) & 64.7 (2.5) & 1650 (12) 
\\ \cmidrule{1-8}
\multirow{6}[2]{*}{0.1} & Decision Tree & 59.8 (4.3)& 59.9 (4.6) & 1911 (202) & 62.5 (1.7) & 62.3 (3.4) & 1782 (83) \\
& Fair Decision Tree & 64.2 (1.2) & 63.5 (3.9) & 1701 (58) & 62.6 (2.1) & 62.6 (2.5) & 1778 (100)\\
& Random Forest & 64.7 (1.3) & 64.5 (2.7) & 1680 (63) & 65.5 (0.2) & 65.5 (2.6) & 1643 (15)\\
& All but CG & 64.9 (1.2) & 64.6 (2.9) & 1649 (80) &65.5 (0.2) & 65.5 (2.6) & 1643 (15)\\
& Fair CG & 65.5 (0.6) & 65.2 (2.4) & 1638 (27) & 65.9 (0.4) & 65.2 (2.7) & 1626 (20)\\
& All &  65.6 (0.6) & 652 (2.1) & 1630 (34) & 66.1 (0.2) & 65.3 (2.6) & 1612 (11)\\ \cmidrule{1-8}
\multirow{6}[2]{*}{0.5} & Decision Tree & 62.6 (4.6)& 62.5 (4.9) & 1777 (220) & 66.6 (0.3) & 66.0 (2.5) & 1586 (15) \\
& Fair Decision Tree & 65.9 (0.4) & 65.8 (3.3) & 1620 (20) & 64.9 (1.6) & 64.1 (2.8) & 1666 (76)\\ 
& Random Forest & 66.0 (0.7) & 65.9 (2.9) & 1615 (31) & 66.5 (0.3) & 65.8 (3.1) & 1590 (15) \\
& All but CG & 66.1 (0.5) & 66.0 (3.9) & 1601 (42) & 66.5 (0.3) & 65.8 (3.1) & 1590 (15) \\
& Fair CG & 66.4 (0.5) & 66.0 (2.4) & 1596 (21) & 66.8 (0.4) & 66.0 (2.5) & 1578 (22) \\
& All &  66.6 (0.6) & 66.3 (2.7) & 1586 (28) & 66.9 (0.3) & 66.1 (2.3) & 1571 (14) \\ 
\bottomrule

\end{tabular}%
}
\end{table*}%

\begin{table*}[h]
\centering
\caption{\label{tab_ruleset_default} Summary of rule set performance on default dataset (standard deviation in parenthesis)}
\setlength{\tabcolsep}{5pt} % Default value: 6pt
%\rotatebox{90}
{\small
\begin{tabular}{c c l l l l l l } \toprule
 &  & \multicolumn{3}{c}{Equality of Opportunity} & \multicolumn{3}{c}{Equalized Odds} \\
$\epsilon$ & Rule Set & Train Acc. & Test Acc. & Hamming loss & Train Acc. & Test Acc. & Hamming loss \\ \midrule
 \multirow{6}[2]{*}{0.025} & Decision Tree & 77.9 (0.6)& 78.1 (0.5) & 5963 (17) & 77.9 (0.1) & 78.2 (0.6) & 5965 (18) \\
& Fair Decision Tree & 77.9 (0.6)& 78.1 (0.5) & 5963 (17) & 77.9 (0.1) & 78.2 (0.6) & 5965 (18) \\
& Random Forest & 77.9 (0.6)& 78.1 (0.5) & 5963 (17) & 77.9 (0.1) & 78.2 (0.6) & 5965 (18) \\
& All but CG &77.9 (0.6)& 78.1 (0.5) & 5963 (17) & 77.9 (0.6)& 78.1 (0.5) & 5963 (17) \\
& Fair CG & 81.4 (0.2) & 80.9 (1.3) & 5601 (30) & 81.4 (0.2) & 80.9 (1.3) & 5601 (30) \\
& All & 82.6 (0.3) & 82.0 (0.7) & 5501 (32) & 82.6 (0.3) & 82.0 (0.7) & 5501 (32) \\ \cmidrule{1-8}
\multirow{6}[2]{*}{0.1} & Decision Tree & 77.9 (0.6)& 78.1 (0.5) & 5963 (17) & 77.9 (0.1) & 78.2 (0.6) & 5965 (18) \\
& Fair Decision Tree & 77.9 (0.6)& 78.1 (0.5) & 5963 (17) & 77.9 (0.1) & 78.2 (0.6) & 5965 (18) \\
& Random Forest & 77.9 (0.6)& 78.1 (0.5) & 5963 (17) & 77.9 (0.1) & 78.2 (0.6) & 5965 (18) \\
& All but CG & 77.9 (0.6)& 78.1 (0.5) & 5963 (17) & 77.9 (0.6)& 78.1 (0.5) & 5963 (17) \\
& Fair CG & 81.4 (0.2) & 80.9 (1.3) & 5601 (30) & 81.4 (0.2) & 80.9 (1.3) & 5601 (30) \\
& All & 82.6 (0.3) & 82.0 (0.7) & 5501 (32) & 82.6 (0.3) & 82.0 (0.7) & 5501 (32)\\ \cmidrule{1-8}
\multirow{6}[2]{*}{0.5} & Decision Tree & 77.9 (0.6)& 78.1 (0.5) & 5963 (17) & 77.9 (0.1) & 78.2 (0.6) & 5965 (18) \\
& Fair Decision Tree & 77.9 (0.6)& 78.1 (0.5) & 5963 (17) & 77.9 (0.1) & 78.2 (0.6) & 5965 (18) \\
& Random Forest & 77.9 (0.6)& 78.1 (0.5) & 5963 (17) & 77.9 (0.1) & 78.2 (0.6) & 5965 (18) \\
& All but CG & 77.9 (0.6)& 78.1 (0.5) & 5963 (17) & 77.9 (0.6)& 78.1 (0.5) & 5963 (17) \\
& Fair CG & 81.4 (0.2) & 80.9 (1.3) & 5601 (30) & 81.4 (0.2) & 80.9 (1.3) & 5601 (30) \\
& All &  82.6 (0.3) & 82.0 (0.7) & 5501 (32) & 82.6 (0.3) & 82.0 (0.7) & 5501 (32) \\
\bottomrule

\end{tabular}%
}
\end{table*}%

	\end{document}